\documentclass[twoside,11pt]{article}

\usepackage{enumitem}
\usepackage[utf8]{inputenc}
\usepackage{listings}
\usepackage{xcolor}
\usepackage{amsmath}
\usepackage{booktabs}
\usepackage{algorithm}
\usepackage{algorithmic}
\usepackage{multirow}
\usepackage{wrapfig}
\usepackage{makecell}
\usepackage{mathtools}
\usepackage{lastpage}
\usepackage{bm}
\usepackage{subfigure}
\usepackage{wrapfig}
\usepackage{amsthm}
\usepackage{url}
\usepackage{textcomp}
\usepackage{natbib}
\usepackage{hyperref}
\usepackage[preprint]{jmlr2e}

\newcommand{\tabincell}[2]{\begin{tabular}{@{}#1@{}}#2\end{tabular}}

\newcommand{\presec}{\vspace{-0em}}
\newcommand{\postsec}{\vspace{-0em}}

\jmlrheading{25}{2026}{1-\pageref{LastPage}}{4/23}{5/24}{23-0537}{Yanyi Li, Yimu Zhang, and Cong Fang}

\ShortHeadings{Principal-Random Subspace for LLM Activation Compression}{Li, Zhang, Fang}
\firstpageno{1}

\begin{document}

\title{PRAC: Principal-Random Subspace for LLM Activation Compression and Memory-Efficient Training}

\author{\name Yanyi Li\boldsymbol{$^{1*}$} \email  yyl0605@foxmail.com \\
    \name Yimu Zhang\boldsymbol{$^{1*}$} \email zym24@stu.pku.edu.cn \\
    \name Cong Fang\boldsymbol{$^{1,2}$} \email fangcong@pku.edu.cn \\
    \addr $^{1}$ School of Intelligence Science and Technology, Peking University, China \\
    \addr $^{2}$ Institute for Artificial Intelligence, Peking University, China \\
    \addr $^{*}$ Equal contribution. \\
    \vspace{-2.5em}
}

\editor{Zeyi Wen}

\maketitle

\begin{abstract}

Activations have become the primary memory bottleneck in large-batch LLM training. However, existing compression methods fail to exploit the spectral structure of activations, resulting in slow convergence or limited compression. To address this, we bridge the relationship between the algorithm’s fast convergence and the requirements for subspace projection, and show that an effective compression should yield an unbiased estimate of the original activation with low variance. We propose \textbf{P}rincipal-\textbf{R}andom Subspace for LLM \textbf{A}ctivation \textbf{C}ompression (\textbf{PRAC}), which novelly decomposes activations into two components: a principal subspace captured via SVD to retain dominant information, and a random subspace sampled from the orthogonal complement to approximate the tail. By introducing a precise scaling factor, we prove that PRAC yields an unbiased gradient estimator with minimum variance under certain conditions. Extensive experiments on pre-training and fine-tuning tasks demonstrate that PRAC achieves up to 36\% total memory reduction with negligible performance degradation and minimal computational cost.  

\end{abstract}

\presec
\section{Introduction}
\postsec
Memory footprint has emerged as a critical bottleneck in the training of large language models (LLMs). During training, the memory overhead primarily consists of three components: model parameters (weights), optimizer states, and activations memory. The activations represent the intermediate outputs of each layer computed during the forward pass, which are typically retained to compute gradients in the subsequent backward propagation. In practice, a moderately large batch size generally promotes more stable training dynamics and improves GPU parallelism utilization, thereby accelerating convergence efficiently. However, a larger batch size substantially increases the activation memory footprint, turning it into the primary bottleneck and limiting both training efficiency and scalability. As illustrated in Figure \ref{fig:first_image} (Right), training a LLaMA-1B model (batch size\footnote{In this paper, ``batch size" denotes the micro-batch size, representing the maximum number of samples processed during a single forward and backward pass on an individual GPU.} of 128 and sequence length of 256) requires 94.5GB of memory, with model parameters 2.48GB, optimizer states and weight gradients 7.95GB for popular used ADAM algorithm \citep{adam2017}, and activation 84.17GB, where activations account for a staggering 89\% of the total usage. 

\begin{figure}[!t]
    \setlength{\tabcolsep}{2.5pt}
    \begin{center}
    \vspace{-0.3cm}
    \begin{tabular}{cccc}
        \includegraphics[height=0.28\linewidth]{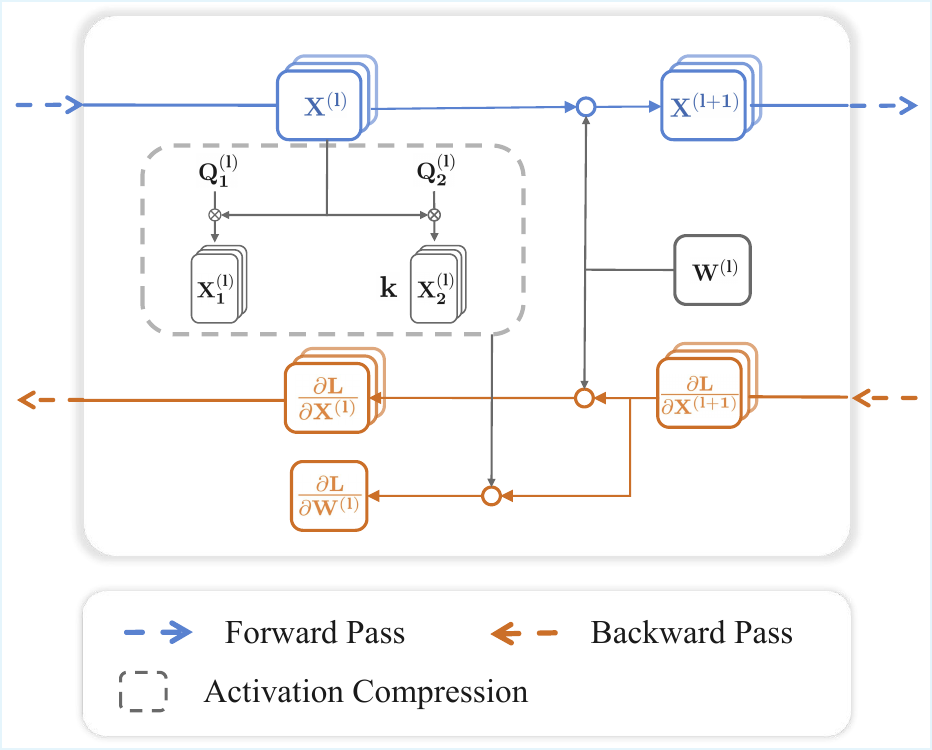}&
        \includegraphics[height=0.28\textwidth]{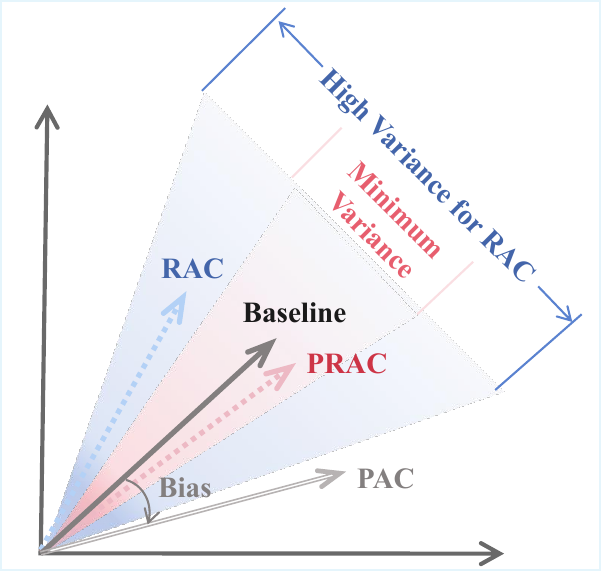}&
        \includegraphics[height=0.28\textwidth]{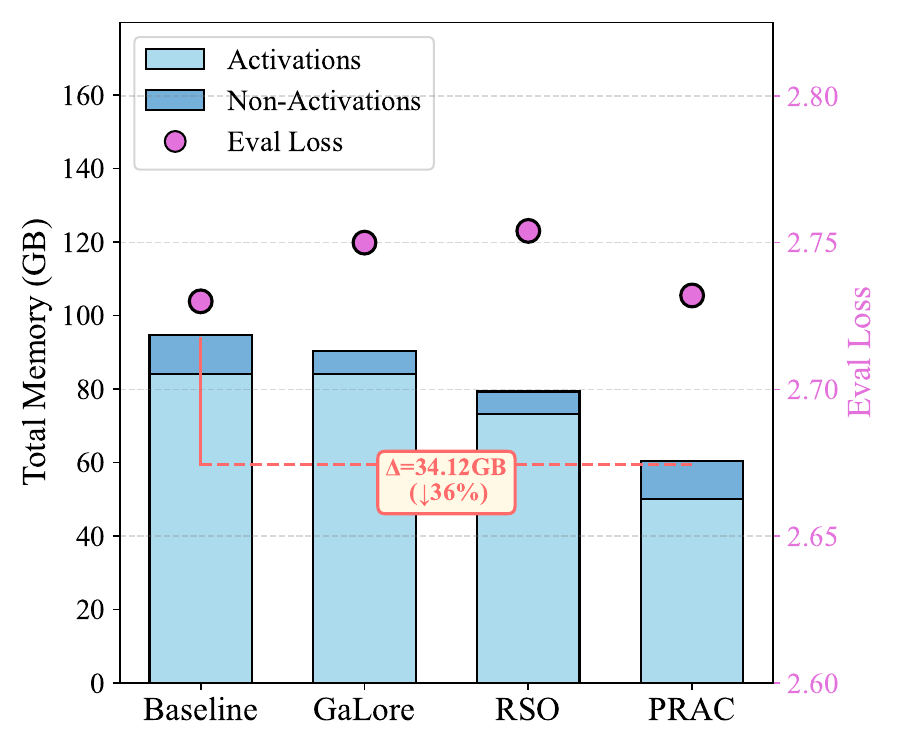}\\
    \end{tabular}
    \end{center}
    \vspace{-0.3cm}
    \caption{The proposed PRAC projects activations onto both the principal and random subspaces and yields the minimum variance unbiased estimator, thus achieving up to 36\% total memory reduction with negligible performance degradation. (Left) The flowchart of PRAC; (Middle) A conceptual comparison of subspace strategies; (Right) Performance comparison on LLaMA-1B.}
    \vspace{-0.4cm}
    \label{fig:first_image}
\end{figure}
A straightforward approach to reducing activation memory is to modify the implementation of the training algorithm while preserving its underlying logic. For instance, gradient checkpointing \citep{checkpoint2016} selectively re-computes certain activations during the back-propagation phase, rather than storing them throughout the forward pass. This technique introduces additional computational overhead in the backward pass, thereby increasing overall training time.

The other line of research focuses on reducing activation memory through novel algorithmic design. The core idea of such approaches is to iteratively optimize within a special chosen subspace of the model parameters.  

One representative example is zeroth-order optimization, which has attracted considerable attention in recent years, leading to multiple proposed variants  \citep{mezo2023,mezo_svrg2024,lozo2024,zo_benchmark2024,hizoo2025,R-adazo2025,muonzo2025}. From an algorithmic perspective, the algorithm per-step samples a random direction $\xi$, typically drawn from a standard normal or uniform spherical distribution, and constructs a gradient estimator of the form
$\frac{1}{\eta}[f (w+ \eta \xi) - f (w)]  \xi$, where $w$ denotes the model parameters and $\eta$ is a small positive constant. Intuitively, as $\eta \to 0$, this estimator approximates $[\xi^\top \nabla f(W)] \xi$, which corresponds to performing an update only along the normalized direction $\xi/\|\xi\|$, scaled by the directional derivative of the objective function. In implementation, such methods only require access to the final output of the forward pass, allowing intermediate activations to be released immediately. This substantially reduces the peak memory footprint during training. However, since these approaches typically update only a one-dimensional subspace per iteration, they converge much more slowly than first-order methods in practice and are generally not suitable for challenging tasks such as large-scale pre-training. 

To improve efficiency, some advanced methods incorporate architectural information to enable multi-dimensional subspace updates. For example, RSO (Random Subspace Optimization) introduced by \citet{rso2025} in the outerloop uniformly samples subspaces and in the inner loop optimizes weights in the subspaces by inserting a low-dimensional trainable module. Activation memory is reduced because only the projected activations need to be stored—a principle inspired by LoRA \citep{lora2021} and ReLoRA \citep{relora2023}. Despite these memory savings, this method still results in non-negligible performance degradation during pre-training, limiting its practical applicability.

We emphasize that a central limitation of existing activation-efficient methods is their inability to leverage the spectral structure of activations, which ultimately compromises training effectiveness. In this paper, we directly focus on  how to  compress activations during training with minimal impact on convergence rate. Following \citet{galore2024,golore2025,mezo2023,mezo_svrg2024}, we still adopt a linear compression scheme: given a projection matrix $P$, the activation matrix $X$ is compressed as $ XP $ and reconstructed as $XPP^\top$.

Our starting point is a standard stochastic optimization formulation, through which we bridge the relationship between the algorithm's fast convergence and the requirements for subspace projection. We demonstrate that a well-chosen projection should yield an unbiased estimate of the original activation with low variance that ensures provable and efficient training convergence.

To design an efficient estimator, we leverage the spectral structure of activations. We formalize widely observed ``low-rank'' phenomenon \citep{lora2021,galore2024,fira2024,adapm2025}  as  the \emph{Activation Degenerate Condition}, which assumes that the singular values of activations consist of a few dominant ones, followed by a long tail of smaller values whose cumulative energy (squared sum) is bounded.

The primary contribution of this paper is the introduction of an activation compression technique, termed PRAC ({\bf P}rincipal-{\bf R}andom Subspace for LLM {\bf A}ctivation {\bf C}ompression). We show this method is optimal in the sense that  it achieves \emph{minimum} variance among all unbiased estimators under the Activation Degenerate Condition.
PRAC consists of two simple components: a principal subspace projection, which obtains a subspace projection matrix $Q_1$  via SVD, and a random subspace projection, where $Q_2$ is sampled uniformly from the orthogonal complement of  $Q_1$. Additionally, a \emph{key} scaling factor $k$ is introduced to ensure the unbiasedness of the estimator.
\[
\tilde{X}=(X P)P^\top =(XQ_1)Q_1^\top+k ( XQ_2)Q_2^\top.
\]

In implemtation for memory-efficient LLM training, PRAC employs a dynamic subspace update schedule that refreshes projection matrices at fixed intervals, thereby rendering the cost of SVD and QR decompositions negligible. The method further maximizes memory efficiency through subspace sharing and a layer-wise policy (see details in Section \ref{subsec:5.1}). We analyze the memory footprint and computational overhead of PRAC, demonstrating its efficiency in both aspects.

We conduct extensive experiments on both pre-training and fine-tuning tasks. Our experiments demonstrate that PRAC matches or surpasses baselines across various tasks, achieving up to 36\% memory saving while maintaining competitive performance.

Our key contributions are summarized as follows:
\begin{itemize}
    \item We propose PRAC, a novel activation compression method that effectively integrates principal and random subspace projections. PRAC is the \emph{first} method to leverage the structural information of activations for memory-efficient pre-training. Theoretically, we prove that PRAC yields an unbiased estimator with \emph{minimum} variance under the Activation Degeneration Condition.
    \item We demonstrate the practical efficacy of PRAC through extensive experiments on both pre-training and fine-tuning tasks. PRAC consistently achieves substantial memory reduction with negligible performance loss and very small computational overhead. 
\end{itemize}

\presec
\section{Related Works}
\postsec
\textbf{Activation-Efficient Methods.} To alleviate activation memory overhead, a technique modifying the checkpointing mechanism during backpropagation was proposed in \citet{checkpoint2016}. This approach has not only achieved widespread adoption but has also catalyzed subsequent research into modifications of the backpropagation algorithm itself \citep{approxbp2024,streambp2025}.

Furthermore, alternative approaches employ zeroth-order optimization to bypass backpropagation entirely \citep{mezo2023,mezo_svrg2024,lozo2024,velora2024,rso2025,compact2025,muonzo2025}, while others retain first-order methods but address the optimization problem within a low-dimensional subspace \citep{rso2025}.

Compression-based approaches \citep{approxbp2024, compact2025, velora2024} offer an alternative strategy by approximating intermediate variables. However, these methods are generally designed for specific sub-components of the model architecture. Furthermore, they fail to explicitly address how compression artifacts influence the convergence rate, limiting their potential in  general-purpose compression for memory-efficient training. 

\noindent\textbf{Optimization States-Efficient Methods.} Considerable research aim to improve the memory efficiency for optimizer states.  GaLore (Gradient Low-Rank Projection) and its variants \citep{galore2024,osd2024,golore2025} perform low-rank compression on gradients. To address the compression of momentum terms in the Adam optimizer, Adafactor \citep{adafactor2018} utilizes factorization techniques to approximate the storage of the second moment. Adam-mini \citep{adam_mini2024} designs a block-wise learning rate strategy based on the heterogeneity of the neural network's Hessian matrix, reducing the memory footprint of the second moment. Fira \citep{fira2024} and Apollo \citep{apollo2025} utilize optimizer states to update adaptive learning rates. However, these algorithms do not modify the backpropagation process and thus cannot reduce the memory usage of activations, which is the dominant factor. Our method is orthogonal to this kind of method and can be combined with them for additional memory savings.





\presec
\section{Starting Point of Activation Compression}\label{sec:starting point}
\postsec
\subsection{Problem Formulation}\label{problem_formulation}
The training of Large Language Models (LLMs) can be formulated as the following stochastic optimization problem:
\begin{equation}\label{optimization_problem}
\min_W f(W)\coloneqq \mathbb{E}_\zeta[F(W;\zeta)],
\end{equation}
where $W$ denotes the model parameters and $\zeta$ represents the random data batch.  We simply study the Stochastic Gradient Descent (SGD) \citep{sgd2010} method, noting that the analysis extends naturally to other optimizers. The update of SGD goes as: 
\begin{equation}\label{sgd}
    W_{t+1}=W_{t}-\eta_t\underbrace{\nabla F(W_t;\zeta_t)}_{\text{update\;term}}.
\end{equation}
To mitigate memory bottlenecks, various methods compress training artifacts (parameters, gradients, or activations). Consequently, the exact gradient is replaced by an approximate update term, denoted as $\tilde{G}(W_t;\zeta_t)$, leading to the modified update rule: 
\begin{equation}\label{update_rule}
    W_{t+1}=W_{t}-\eta_t \tilde{G}(W_t;\zeta_t).
\end{equation}

A fundamental question arises: \textbf{What  properties should the compression satisfy to guarantee fast convergence?}

To address this, we first examine the impact of systematic bias in the gradient estimator.

\begin{proposition} [Non-convergence under Constant Bias]\label{proposition_1}
   Consider two-dimensional strongly convex problem:  $f(w^{(1)},w^{(2)}) =\mathbb{E}_{\xi'}(w^{(1)}- 3\xi')^2+(w^{(2)})^2$ with $\xi'$ following a Bernoulli distribution. The parameters are initialized at  $w^{(1)} = 0, w^{(2)} = 1$. At each step, only the first coordinate is updated by a stochastic gradient:  $\tilde{G} = [ 2(w^{(1)}-3\xi'),0]^\top.$  Then for any $\left\{\eta_t\right\}$ satisfies the Robbins-Monro condition: $\sum_{t=0}^{\infty}\eta_t=\infty,\sum_{t=0}^{\infty}\eta_t^2<\infty$, one has $w^{(2)}_t =1$ for all $t$. Consequently,  $(w^{(1)}_t, w^{(2)}_t)$ will never converge to an approximate stationary point.
\end{proposition}

In the simple example above, the gradient estimator is biased by $\Theta(1)$ because $w^{(2)}$ is never updated. It is interesting to  observe that when $|w^{(1)}| < 2$, the magnitude of the stochastic gradient for $w^{(1)}$ (i.e., $2|w^{(1)}-3\xi'|$) is always larger than that for $w^{(2)}$ (i.e., $2|w^{(2)}|$). Consequently, even if $w^{(1)}$ converges to its minimizer $0$, the maximum selection rule that consistently picks the coordinate with the largest gradient magnitude will never update $w^{(2)}$. It implies that principal methods, such as Galore \citep{galore2024}, may fail to converge in general stochastic optimization settings. Proposition \ref{proposition_1}  also establishes that a constant biased estimator cannot ensure convergence.  We then consider unbiased  estimator with non-negligible variance.

\begin{proposition}[Convergence Rate with Bounded Gradient Variance, \citet{sgd_pre1}]\label{proposition_2}
    Let the objective function $f$ is L-smooth and bounded below, and denote $\Delta=f(W_1)-\inf{f}$. Assume that at each iteration $t$ we have access to an unbiased stochastic gradient $\tilde{G}(W_t,\zeta_t)$ satisfying
    \(
    \mathbb{E}[\tilde{G}(W_t,\zeta_t)\mid W_t]=\nabla f(W_t),\mathbb{E}[\|\tilde{G}(W_t,\zeta_t)-\nabla f(W_t)\|^2 \mid W_t]\leq \sigma^2.
    \)
    Consider the update rule in \eqref{update_rule} with constant step sizes $\eta_t=\min\left\{\frac{1}{L},\frac{\sqrt{2\Delta/L}}{\sigma\sqrt{N}}\right\}$ and then randomly selects an output $W_R$ from $\{W_1,\cdots,W_T\}$ according to the certain probability mass function. Then the following bound holds:
    \[
    \mathbb{E}[\|\nabla f(W_R)\|^2]\leq \frac{2L\Delta}{T}+2\sqrt{\frac{2L\Delta}{T}}\sigma.
    \]
\end{proposition}

Proposition \ref{proposition_2}  is a standard result in non-convex optimization \citep{sgd_pre2, sgd_pre1}. It 
demonstrates that  unbiased compression ensures convergence. Moreover, the convergence rate is hindered by the variance $\sigma^2$. Therefore, an effective compression should yield a gradient estimator that is both \textbf{unbiased} and \textbf{low-variance}.

Since activations dominate memory usage in LLM training, we translate the gradient-level requirements established above into concrete criteria for activation compression.

\subsection{Criteria for Activation Compression}\label{activation compression}

The theoretical constraints established in Section \ref{problem_formulation} for general gradients directly inform the design of activation compression. For a linear layer, such as the Query layer in the attention block, the quality of the activation reconstruction dictates the quality of the gradient estimate.

\begin{lemma}\label{lemma:grad_est_properties}
Consider a linear layer with forward pass $Y = XW$. Let $\tilde{X}$ be a compressed estimator of the activation $X$. If $\tilde{X}$ is unbiased ($\mathbb{E}[\tilde{X}]=X$) with bounded variance $\mathbb{E}[\|\tilde{X}-X\|_F^2]\leq \sigma^2$, then the resulting gradient estimator $\tilde{G} = \tilde{X}^\top (\nabla_Y L)$ satisfies:
\begin{enumerate}
\item \textbf{Unbiasedness:} The gradient estimator is unbiased with respect to the true gradient:
\[\mathbb{E}[\tilde{G}]=\mathbb{E}[\tilde{X}^\top] (\nabla_Y L)=\nabla_W L.\]
\item \textbf{Bounded Variance:} The gradient variance is bounded by the activation variance and the upstream gradient norm:
\[
\mathbb{E}[\|\tilde{G}-\nabla_W L\|_F^2]\leq\sigma^2\|\nabla_YL\|_2^2.
\]
\end{enumerate}
\end{lemma}

Lemma \ref{lemma:grad_est_properties} shows that constructing an unbiased and low-variance approximation of activations suffices to satisfy the convergence guarantees in Propositions \ref{proposition_1} and \ref{proposition_2}.

Following \citet{galore2024,golore2025,mezo2023,mezo_svrg2024}, we adopt a subspace projection approach. Let $P\in\mathbb{R}^{n\times r}$ be a projection matrix, and the activation $X$ is compressed to a lower-dimensional representation $XP$ and reconstructed as $\tilde{X}=(XP)P^\top$. The central challenge, therefore, lies in designing the projection matrix $P$ such that the reconstruction $\tilde{X}$ minimizes variance while maintaining unbiasedness, specifically tailored to the spectral structure of the activations.

\presec
\section{Proposed Method: PRAC}\label{sec:proposed_prac}
\postsec

In this section, we introduce PRAC, a hybrid framework motivated by the spectral structure of activations. By integrating the principal and random subspaces to address their respective limitations, we balance the bias-variance trade-off. Finally, we prove that PRAC yields the minimum variance unbiased estimator under certain conditions.


\subsection{Spectral Analysis of Activations}

\begin{wrapfigure}{r}{0.5\linewidth}
    \centering
    \vspace{-0.1cm}
    \includegraphics[width=\linewidth]{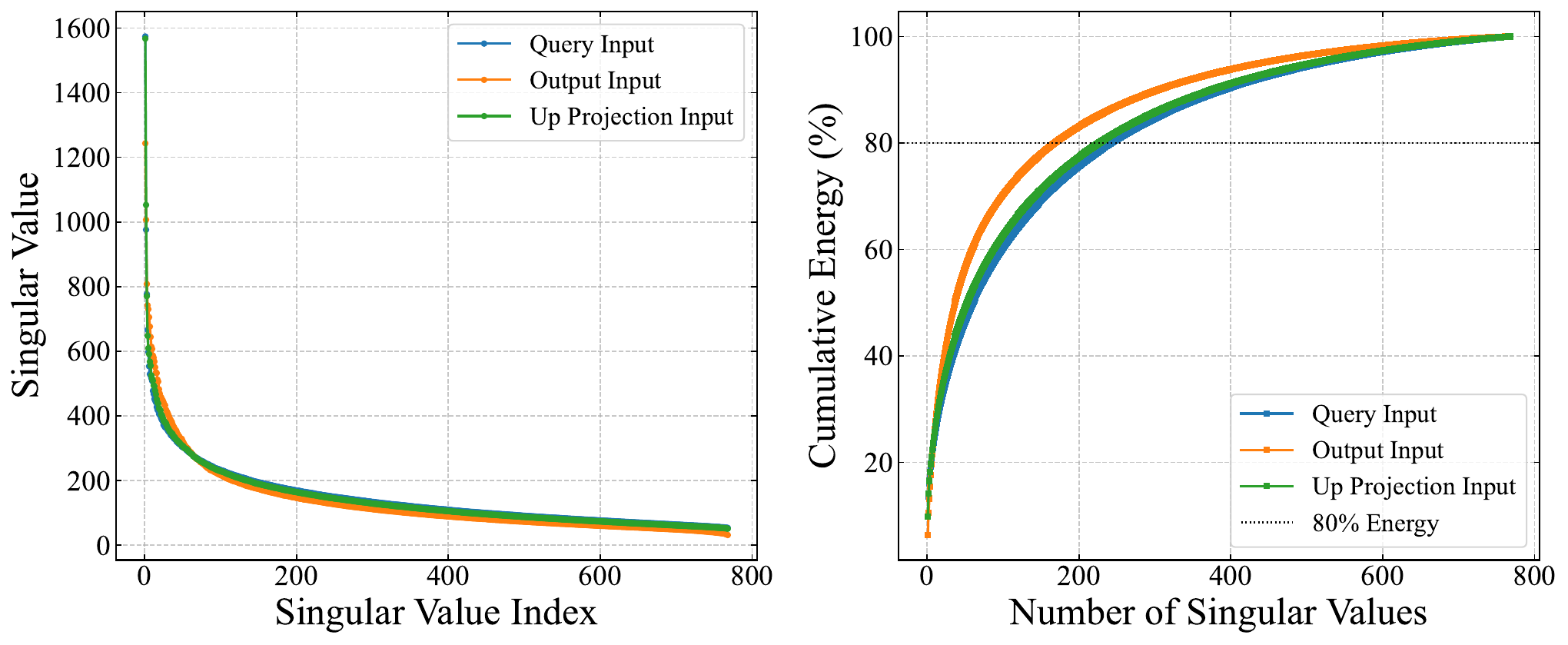} 
    \caption{Singular value spectrum (Left) and cumulative energy ratio (Right).}
    \label{fig:svd}
    \vspace{-0.6cm}
\end{wrapfigure}

To design an effective projection matrix $P$ in Section \ref{activation compression}, we analyze the spectral structure of activations using the 10th layer of LLaMA-130M (at 10\% training progress) as an example. 

As shown in Figure \ref{fig:svd}, the spectrum exhibits two distinct characteristics: a few dominant singular values capture the majority of the spectral energy; and the remaining singular values form a long tail that decays slowly.

Note the ``low-rank’’ structure is commonly observed in weight matrices \citep{lora2021,relora2023}. We study activation matrix and emphasize that the long-tail directions cannot be ignored—truncating them significantly slows down convergence \citep{fira2024,adapm2025}. We introduce the term \emph{degenerate condition} to characterize this observed phenomenon.

\begin{assumption}[Activation Degenerate Condition]\label{ass:spectral_decay}  For activation matrix $X$,  let $\sigma_i$  be the $i$-th largest singular value of the activations.  We say $X$ satifies the  $(s, q)$-degenerate condition ($X\sim(s,q)-D$ in short) if it admits: $\sum_{s+1}^n \sigma_i^2 \leq q$.
\end{assumption}
This assumption posits that only the top $s$ singular values of the activation matrix are large. Since the remaining singular values are typically small and roughly equal in practice, we bound the rest squared sum (same as the square of the Frobenius norm for the rest matrix) by a constant $q$. 

Based on Assumption \ref{ass:spectral_decay}, we can conduct a theoretical analysis of the optimal design of the projection matrix $P$ (see Section \ref{sec:prac}). Before that, let us first introduce two key components of PRAC, which are also subspace projection methods.

\subsection{Key Components of PRAC}\label{sec:key components}
Intuitively, one might consider two natural approaches to activation compression: projecting onto the principal subspace of the activations, or onto a random subspace. However, both have inherent drawbacks that our method addresses.

\noindent\textbf{Principal Subspace for Activation Compression (PAC).}
PAC retains the significant information by projecting activations onto their principal subspaces. Let the Singular Value Decomposition (SVD) of the activation matrix be $X=U\Sigma V^\top=\sum_{i=1}^n s_iu_iv_i^\top$. We define the projection matrix $Q_1=\left[v_1^{(1)},\cdots,v_{r_1}^{(1)}\right]\in \mathbb{R}^{n\times r_1}$ using the top-$r_1$ right singular vectors. Then $XQ_1$ captures the rank-$r_1$ principal information. For the rank $r_1 < n$, PAC introduces a systematic reconstruction bias $\Delta=\|X- XQ_1Q_1^\top\|_F^2=\sum_{i=r_1+1}^{n}\sigma_i^2>0$. As discussed in Section \ref{activation compression}, this bias can lead to non-convergence in the optimization problem (\ref{optimization_problem}).

\noindent\textbf{Random Subspace for Activation Compression (RAC).} To avoid bias, RAC projects activations onto a random subspace sampled uniformly from the Stiefel manifold $\mathrm{St}_{n,r_2}$, defined as  $\mathrm{St}_{n,r}=\left\{P\in \mathbb{R}^{n\times r} |P^\top P=I_r\right\}$. The construction proceeds as follows: First, we generate a random matrix $S\in \mathbb{R}^{n\times r_2}$ with entries $S_{ij}\sim \mathcal{N}(0,1)$. We then perform QR decomposition on $S$ to obtain an orthogonal basis $Q_2\in \mathbb{R}^{n\times r_2}$, such that $Q_2^\top Q_2=I$. 

While RAC can provide an unbiased estimator of the activation (via setting a scaling factor), the inherent randomness introduces high variance: $ \mathbb{E}\|\tilde{X}_{\text{rac}} -X \|_F^2 \leq (\frac{n}{r_2}-1)\left(\sum_{i=1}^s \sigma_i^2+q\right)$ due to the top singular values. (see Theorem \ref{thm:unbiased_reconstruction_of_rac_} and \ref{thm:gradient_variance_of_rac_} in Appendix \ref{omit_2}). This high variance destabilizes the training process and slows convergence.

\subsection{Optimal Design via Hybrid Projection} \label{sec:prac}
We now establish the lower bound variance to estimate the activations. Specifically, for any activation $X$, we consider the random subspace projection method where the projection matrix $P \in \mathbb{R}^{n \times r}$ is drawn from a distribution $\mathbb{P}_X$, which must satisfy the unbiasedness condition (i.e. $\mathbb{E}_{P\sim\mathbb{P}_X}[XPP^\top]=X$) and depend on $X$, then the lower bound variance can be written as 
\begin{equation}\label{eq:lower bound}
 \sup_{X \sim  (s,q)-D  } \left(\inf_{\substack{
    P\sim \mathbb{P}_X, \\
    \mathbb{E}_{\mathbb{P}_X}[XPP^\top] = X}} \mathbb{E}_{\mathbb{P}_X} \|X PP^\top -X \|_F^2\right).   
\end{equation}

For this min-max optimization problem, if the small singular values of $X$ exhibit non-uniformity, $P$ can be tailored to exploit this structure. Intuitively, a distribution where $X$ possesses approximately uniform tail singular values represents a relatively worst-case scenario. The following lemma gives the lower bound explicity.

\begin{lemma}[Lower Bound]\label{lemma:lower}
Under Assumption \ref{ass:spectral_decay}, when $s< r< n$, we have $\eqref{eq:lower bound} \geq (\frac{n-s}{r-s}-1)q$. 
\end{lemma}

We now propose our method and demonstrate that its variance achieves the lower bound in Lemma \ref{lemma:lower}, thereby demonstraining its optimality in attaining the minimum variance.

\noindent\textbf{Construction of the Projection Matrix.} We first extract the principal basis vectors $v_i^{(1)}(i\in[1,r_1])$ via SVD, and let the projection matrix $Q_1=\left[v_1^{(1)},\cdots,v_{r_1}^{(1)}\right]$. Then we sample a random matrix  whose columns lie in the orthogonal complement of the column space of $Q_1$: \begin{equation}\label{eqn:so_prac}S_o=(I-Q_1Q_1^\top)S, \quad \text{where} \;S_{ij}\stackrel{i.i.d}{\sim} \mathcal{N}(0,1). \end{equation} Let $v_j^{(2)}(j\in [1,r_2])$ be the column vectors obtained from $\mathrm{QR}(S_o)$. By construction, the principal and random subspaces are orthogonal ($v_i^{(1)}\perp v_j^{(2)}$). The unified projection matrix $P\in \mathbb{R}^{n\times (r_1+r_2)}$ is defined as: \begin{equation}\label{eqn:R} P = \left[\underbrace{v_1^{(1)},\cdots,v_{r_1}^{(1)}}_{\mathrm{Principal}}, \underbrace{\sqrt{k}v_{1}^{(2)},\cdots,\sqrt{k}v_{r_2}^{(2)}}_{\mathrm{Random}}\right], \end{equation} 

where $k$ is a scaling factor used to ensure unbiased reconstruction. It's critical, since we only sample a subspace of dimension $r_2$ to represent the entire tail of dimension $n-r_1$, the energy of the random component must be amplified to ensure the estimator remains unbiased in expectation.

\noindent\textbf{Unbiased Estimation.} By selecting an appropriate $k$, PRAC ensures the unbiasedness.


\begin{theorem}[Unbiased Reconstruction of PRAC]\label{thm:unbiased_reconstruction} Let $P$ be the projection matrix defined in \eqref{eqn:R}. If the scaling factor is set to $k=\frac{n-r_1}{r_2}$, the reconstruction $\tilde{X}=XPP^\top$ satisfies $\mathbb{E}[\tilde{X}]=X$. \end{theorem}

\noindent\textbf{Variance Reduction Analysis.} A key advantage of PRAC is its ability to  achieve the minimum  variance. 

\begin{theorem}[Variance Bound of PRAC]\label{thm:gradient_variance}
For $X$ admiting $(s,q)$-D condition with $s<r<n$, by choosing $r_1=s$ and $r_2 = r-r_1$, we have: 
\[
\mathbb{E}\bigl[\|XPP^\top - X\|_F^2\bigr] \leq (\frac{n-r_1}{r_2}-1) q.
\]\end{theorem}

Theorem \ref{thm:unbiased_reconstruction} confirms that setting $k=\frac{n-r_1}{r_2}$ yields an unbiased reconstruction of activations and, by extension, gradients (Lemma \ref{lemma:grad_est_properties}). Theorem \ref{thm:gradient_variance} demonstrates its optimality in the sense
that it achieves the minimum variance among all unbiased estimators under the degenerate condition. Ablation study in Section \ref{sec:ablation} empirically validates this theoretical finding. A comparison of the three projection methods is presented in Table \ref{tab:subspace}.

\presec
\section{PRAC for Memory-Efficient Training}\label{sec:5}
\postsec

In this section, we detail the practical implementation of PRAC for memory-efficient LLM training. We introduce dynamic and sharing strategy to minimize computational overhead. Furthermore, we present a layer-wise adaptation scheme tailored to the distinct sensitivity of different model components. Finally, we provide a complexity analysis of PRAC, focusing on activation memory reduction and the amortized computational cost.

\begin{table}[!t]
    \caption{Comparison of Estimation Unbiasedness and Variance across Projection Methods.}
    \vspace{-0.2cm}
    \label{tab:subspace} 
    \setlength{\tabcolsep}{1.5pt}
    \begin{center}
    \begin{small}
    \begin{tabular}{lcccc}
    \toprule
    Algorithm & Form & Scaling & Unbiasedness & Variance\\
    \cmidrule{1-5}
    PAC & $\tilde{X}_{\text{pac}}=(XQ_1)Q_1^\top$ & - & $\times$ & - \\
    RAC & $\tilde{X}_\text{rac}=(XQ_2)Q_2^\top$ & $n/r_2 $ & $\checkmark$ & High\\
    PRAC & $\tilde{X}_{\text{prac}}=(XP)P^\top$ & $(n-r_1)/r_2$ &$\checkmark$ & Minimum \\
    \bottomrule
    \end{tabular}
    \end{small}
    \end{center}
    \vspace{-0.2cm}
\end{table}

\subsection{PRAC for LLM Training}\label{subsec:5.1}
\noindent\textbf{Dynamic Subspace Update.} Recall that the PRAC reconstruction can be written as:
\begin{equation}
    \begin{aligned}
        XPP^\top &= X\left[\sum_{i=1}^{r_1}v_i^{(1)} (v_i^{(1)})^{\top}+k\sum_{i=1}^{r_2}v_i^{(2)} (v_i^{(2)})^{\top}\right] =XQ_1Q_1^\top + kXQ_2Q_2^\top,
    \end{aligned}
\end{equation}
where $Q_1=[v_1^{(1)},\cdots,v_{r_1}^{(1)}],Q_2=[v_1^{(2)},\cdots,v_{r_2}^{(2)}]$ are the principal and random projection matrices respectively.

Computing the exact principal subspace (via SVD) and generating orthogonal random projections (via QR) at every step incurs significant computational cost. To mitigate this, we employ a lazy update strategy. The principal and random components $Q_1,Q_2$ are updated only at fixed intervals $T_1$ and $T_2$, respectively. The rationale is that the activation statistics (and thus the related subspace) change slowly over training steps, especially after the initial warm-up phase. This lazy update mechanism drastically reduces the amortized computational overhead. The full procedure is outlined in Algorithm \ref{alg:pq}.

\noindent\textbf{Subspace Sharing.}  In Transformer architectures, many layers share the same input activation $X$ such as the Query, Key and Value projections. To cut memory usage, we enforce subspace sharing: a single set of projection matrices ($Q_1,Q_2$) and compressed activations ($XQ_1,XQ_2$) is computed and shared across these layers. Beyond memory savings, this strategy ensures that gradient updates for parallel heads reside in a consistent subspace, reducing gradient noise variance and stabilizing optimization.

\begin{algorithm}[!t]
    \caption{Dynamic Activation Compression via PRAC}
    \label{alg:pq}
    \begin{small}
    \begin{algorithmic}[1]
        \REQUIRE Input activation $X$, ranks $(r_1, r_2)$, update intervals $(T_1, T_2)$, current step $t$
        \ENSURE Projection matrices $Q_1^{(t)}$, $Q_2^{(t)}$, Compressed Activations $X_1^{(t)}$, $X_2^{(t)}$
        
        \STATE $k \gets (n - r_1)/r_2$
        
        \STATE \textit{// Update Principal Subspace:}
        \IF{$t \bmod T_1 = 0$}
            \STATE $Q_1^{(t)} \gets \text{top-}r_1\text{ right singular vectors of } X$
        \ELSE
            \STATE $Q_1^{(t)} \gets Q_1^{(t-1)}$
        \ENDIF
        
        \STATE \textit{// Update Random Subspace}
        \IF{$t \bmod T_2 = 0$}
            \STATE Sample $S \sim \mathcal{N}(0, I_{n \times r_2})$
            \STATE $Q_2^{(t)} \gets \text{orthonormal basis of } (I - Q_1^{(t)}Q_1^{(t)\top})S$
        \ELSE
            \STATE $Q_2^{(t)} \gets Q_2^{(t-1)}$
        \ENDIF
        
        \STATE \textit{// Compression}
        \STATE $X_1^{(t)} \gets XQ_1^{(t)}$, 
               $X_2^{(t)} \gets k \cdot X Q_2^{(t)}$
        
        \STATE \textbf{Return:} $Q_1^{(t)}, Q_2^{(t)}, X_1^{(t)}, X_2^{(t)}$
    \end{algorithmic}
    \end{small}
\end{algorithm}
\noindent\textbf{Layer-Wise Design.} Building on the periodic update strategy, we further tailor the compression to the specific roles of different layer types, as not all layers contribute equally to training dynamics. To this end, we adopt a layer-wise configuration tailored to the specific characteristics of linear and non-linear modules:
\begin{itemize} 
    \item \textbf{Linear Layers (MLP, Attention Projections):} Due to their large parameter scale, the updates of their weight matrices impose a higher demand on gradient quality. Therefore, more activation information needs to be retained to ensure optimization performance. During compression, we allocate a relatively higher rank to these layers (e.g., $r_1=r_2=\lfloor0.3n\rfloor$).
    \item \textbf{Non-linear Layers (LayerNorm, RMSNorm, GeLU, SiLU):} These layers typically involve element-wise operations or vector-wise scaling parameters, resulting in lower information density. Although our theoretical analysis does not directly cover all cases, we find that our method remains effective for most layers. It is worth noting that certain intermediate variables in these layers, such as the mean and variance in LayerNorm, are not compressed due to their negligible memory footprint (e.g., $bs$ vs. $bsn$). Likewise, operations like Flash Attention are excluded from PRAC compression because of their intricate coupled structure and engineering optimizations. In practice, we process all the selected non-linear layers using a lower rank ($r_1=r_2=\lfloor0.2n\rfloor$).
\end{itemize}
\begin{table*}[!t]
    \caption{Comparison with memory-efficient algorithms on pretraining tasks. Validation perplexity (PPL, lower is better) and peak memory consumption (MEM, in GB, lower is better) are reported. For methods marked out-of-memory (OOM), perplexity is measured using half the batch size specified above. $B,S,\Delta$ denote the micro-batch size, sequence length and total memory reduction respectively.}
    \vspace{-0.2cm}
    \label{pretrain_llama} 
    \setlength{\tabcolsep}{0.5pt}
    \begin{center}
    \begin{small}
    \begin{tabular}{lccccccccccccccc}
    \toprule
    &\multicolumn{3}{c}{LLaMA-130M}&\multicolumn{3}{c}{LLaMA-350M}&\multicolumn{3}{c}{LLaMA-1B}&\multicolumn{3}{c}{GPT-2-124M}&\multicolumn{3}{c}{GPT-2-355M}\\
    \cmidrule{1-16}
      Method &PPL &MEM &$\Delta$ \;\;&PPL &MEM &$\Delta$ \;\; &PPL &MEM &$\Delta$ \;\;&PPL &MEM &$\Delta$ \;\; &PPL &MEM &$\Delta$ \;\; \\
    \cmidrule{1-16}
    Baseline    &24.41&21.26 & -
             &18.77 & 46.10 & - 
             &15.33 &\multicolumn{2}{c}{OOM} 
             &19.83 & 62.27 & -
             &16.26 & 55.94 & - \\
    \cmidrule{1-16}
    GaLore   & 25.12 & 21.08 & -1\%
             & 19.65 & 45.37 & -2\% 
             & 15.64&\multicolumn{2}{c}{OOM} 
             & 21.01 & 62.15 & -0\% 
             & 17.88 & 55.64 & -1\% \\
    RSO      & 25.41 & 19.57 & -8\%
             & 19.57 & 40.36 & -13\% 
             & 15.72 & 79.32 & -16\%
             & 22.31 & 54.79 & -12\% 
             & 18.29 & 46.90 & -16\% \\
    PRAC     & \textbf{24.66} & \textbf{15.65} & \textbf{-27\%}
             & \textbf{18.92} & \textbf{32.21} & \textbf{-30\%} 
             & \textbf{15.41} & \textbf{60.48} & \textbf{-36\%}
             & \textbf{19.95} & \textbf{49.68} & \textbf{-20\%}
             & \textbf{16.35} & \textbf{44.21} & \textbf{-21\%} \\
    \cmidrule{1-16}
    B / S & \multicolumn{3}{c}{128/256} &\multicolumn{3}{c}{128/256} & \multicolumn{3}{c}{128/256} & \multicolumn{3}{c}{64/1024} &\multicolumn{3}{c}{32/1024} \\
    \bottomrule
    \vspace{-0.6cm}
    \end{tabular}
    \end{small}
    \end{center}
\end{table*}
\subsection{Memory and Computational Efficiency}
\begin{wraptable}{r}{0.5\linewidth}
    \centering
    \small
    \vspace{-0.6cm}
    \caption{Activation memory footprint comparison for the GELU and linear layer.}
    \label{tab:ffn1}
    \begin{tabular}{@{}l@{}c@{}c@{}}
        \toprule
        & \text{Baseline} & \text{PRAC}  \\
        \midrule
        $\text{GeLU \;Input}$ & $bsn$ & $bs(r_1+r_2)$\\
        $W_\text{down} \; \text{Input}$ & $bsn$ & $bs(r_1+r_2)$\\
        Total & $2bsn$ & $2bs(r_1+r_2)$ \\
        \bottomrule
    \end{tabular}
    \vspace{-0.3cm}
\end{wraptable}
\textbf{Theoretical Memory Footprint.} We analyze memory usage in the GeLU layer $A_u=\text{GeLU}(A_u^\prime)$ and linear layer $A_d^\prime = A_uW_{\text{down}}^\top$. Standard backpropagation requires storing full tensors $A_u,A_u^\prime\in\mathbb{R}^{b\cdot s \times n}$, costing $2bsn$. In contrast, PRAC projects these activations onto principal ($r_1$) and random ($r_2$) subspaces, caching only the low-dimensional projections in $\mathbb{R}^{b\cdot s \times (r_1+r_2)}$. In particular, since $P_u,P_u^\prime$ are independent of the batch size, their memory overhead is negligible. The memory usage is summarized in Table \ref{tab:ffn1}.
  

In the experiments, $r_1+r_2$ is typically set to less than $\lfloor 0.6n\rfloor$, which results in over 40\% reduction in the activations. The complete analysis for the full GPT architecture is provided in Appendix \ref{sec:complexity}.

\noindent\textbf{Computational Overhead.} The primary computational cost of PRAC arises from SVD and QR decompositions. However, by setting the update intervals $T_1$ and $T_2$ to sufficiently large values (e.g., $T_1=T_2=500$ steps), the amortized cost becomes negligible. Furthermore, the additional matrix multiplications introduced by the projection operation $XP$ incur minimal overhead compared to the memory throughput gains. The experimental results of training efficiency are shown in Section \ref{sec:pretrain}.

\presec
\section{Experiments}
\postsec
In this section, we extensively evaluate PRAC on both pre-training and fine-tuning tasks across various model architectures. PRAC consistently achieves up to 36\% memory reduction while maintaining competitive performance.

\subsection{Memory-Efficient Pre-training}\label{sec:pretrain}
\noindent\textbf{Setup.} We evaluate PRAC by pre-training LLaMA \citep{LLaMA} models (130M, 350M, 1B parameters) on the C4 dataset \citep{C4} and GPT-2 \citep{GPT-2} models (124M, 355M) on OpenWebText \citep{OpenWebText}, following the standard configurations in the nanoGPT codebase\footnote{\url{https://github.com/karpathy/nanoGPT}}. For PRAC, we configure the subspace ranks as $r_1=r_2=\lfloor0.3n\rfloor$ for linear layers and $r_1=r_2=\lfloor0.2n\rfloor$ for non-linear layers (e.g., RMSNorm, SiLU), where $n$ is the hidden dimension. We compare against two memory-efficient baselines: GaLore \citep{galore2024} and RSO \citep{rso2025}, using their reported hyperparameters. Detailed configurations are provided in Appendix \ref{sec:Pre-training Setting}. To ensure statistical robustness, all experiments are conducted over multiple runs with different random seeds, and we report the averaged results.

\begin{wrapfigure}{r}{0.5\linewidth}
\centering
\scriptsize
\vspace{-0.2cm}
\begin{tabular}{@{}c@{}c@{}}
\includegraphics[width=0.5\linewidth]{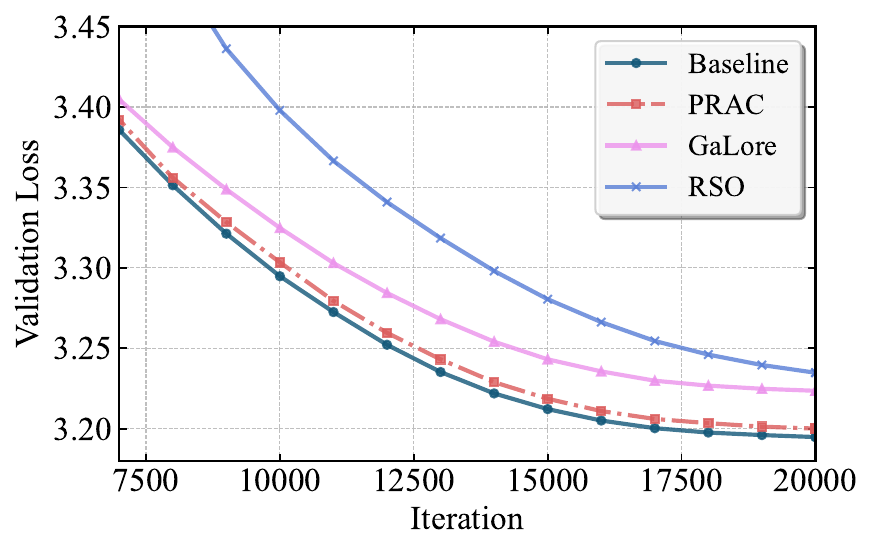} &
\includegraphics[width=0.5\linewidth]{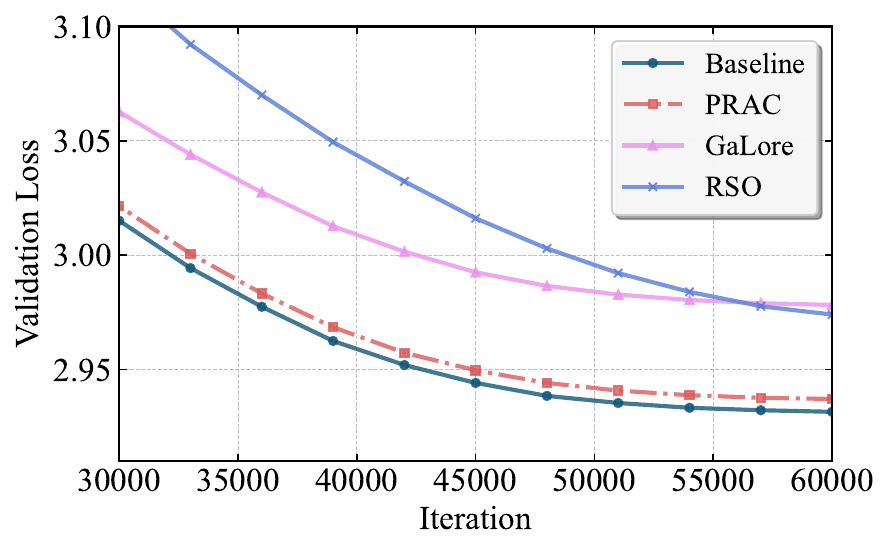} \\
(a) LLaMA 130M & (b) LLaMA 350M
\end{tabular}
\caption{Loss curves of pre-training LLaMA-130M and LLaMA-350M model.}
\label{fig:pretrain}
\vspace{-0.4cm}
\end{wrapfigure}
\noindent\textbf{Results.} Table \ref{pretrain_llama} summarizes the quantitative results. PRAC achieves the optimal trade-off between training performance and memory usage, achieving a 36\% total memory reduction for the LLaMA-1B model. Figure \ref{fig:pretrain} illustrates the validation loss trajectories. PRAC maintains competitive convergence throughout the entire training process. In contrast, GaLore (Principal-only) suffers from stagnation in the later stages due to bias, while RSO (Random-only) exhibits slower initial convergence due to high variance. In the GPT-2 experiments, we restrict the batch size to avoid OOM failures in the comparison methods; under these conditions, our method achieves 20$\times$ and 21$\times$ total memory compression for GPT-124M and 355M respectively.

\begin{wraptable}{r}{0.5\linewidth}  
    \centering
    \small
    \vspace{-0.2cm}
    \caption{Training Time Comparison on $4\times$A800 GPU for LLaMA-1B.}
    \label{tab:throughput}
    \begin{tabular}{@{}l@{}c@{}c@{}c@{}}
    \toprule
    Method & Batch Size & Memory (GB) & Training Time (h) \\
    \cmidrule{1-4}
    Baseline & 64 & 50.35 & 156 \\
    \cmidrule{1-4}
    \multirow{2}{*}{PRAC}  & 64 & 36.90 & 179 \\
    & 96 ($\bf{\uparrow 50\%}$) & 48.70 & 117 ($\bf{\downarrow 25\%}$) \\ 
    \bottomrule
    \end{tabular}
\end{wraptable}

\noindent\textbf{Training Efficiency.} To verify the practical speedup of PRAC, we measure training time and peak memory on LLaMA-1B using 4×A800 GPUs. As shown in Table \ref{tab:throughput}, PRAC reduces peak memory by approximately 27\% at a fixed batch size of 64 (to prevent OOM for the baseline). Crucially, this memory headroom allows for a larger batch size (increasing from 64 to 96), which reduces the total training time by 25\%.

\noindent\textbf{Scaling Laws.} We validate the scalability of PRAC by training a LLaMA series ranging from 35M to 1B parameters. Following Chinchilla's law \citep{scaling2022}, we set the token budget to $20\times$ the parameter count. Figure \ref{fig:scaling_law}(a) shows that PRAC tracks the AdamW baseline loss curves almost identically across all scales while using around 30\% less memory. The linear fit in Figure \ref{fig:scaling_law}(b) confirms that PRAC adheres to the power-law scaling relationship, suggesting robust performance for even larger models.

\begin{figure}[!ht]
    \scriptsize
    \setlength{\tabcolsep}{0.5pt}
    \vspace{-0.2cm}
    \begin{center}
    \begin{tabular}{ccc}
        \includegraphics[height=0.23\textwidth]{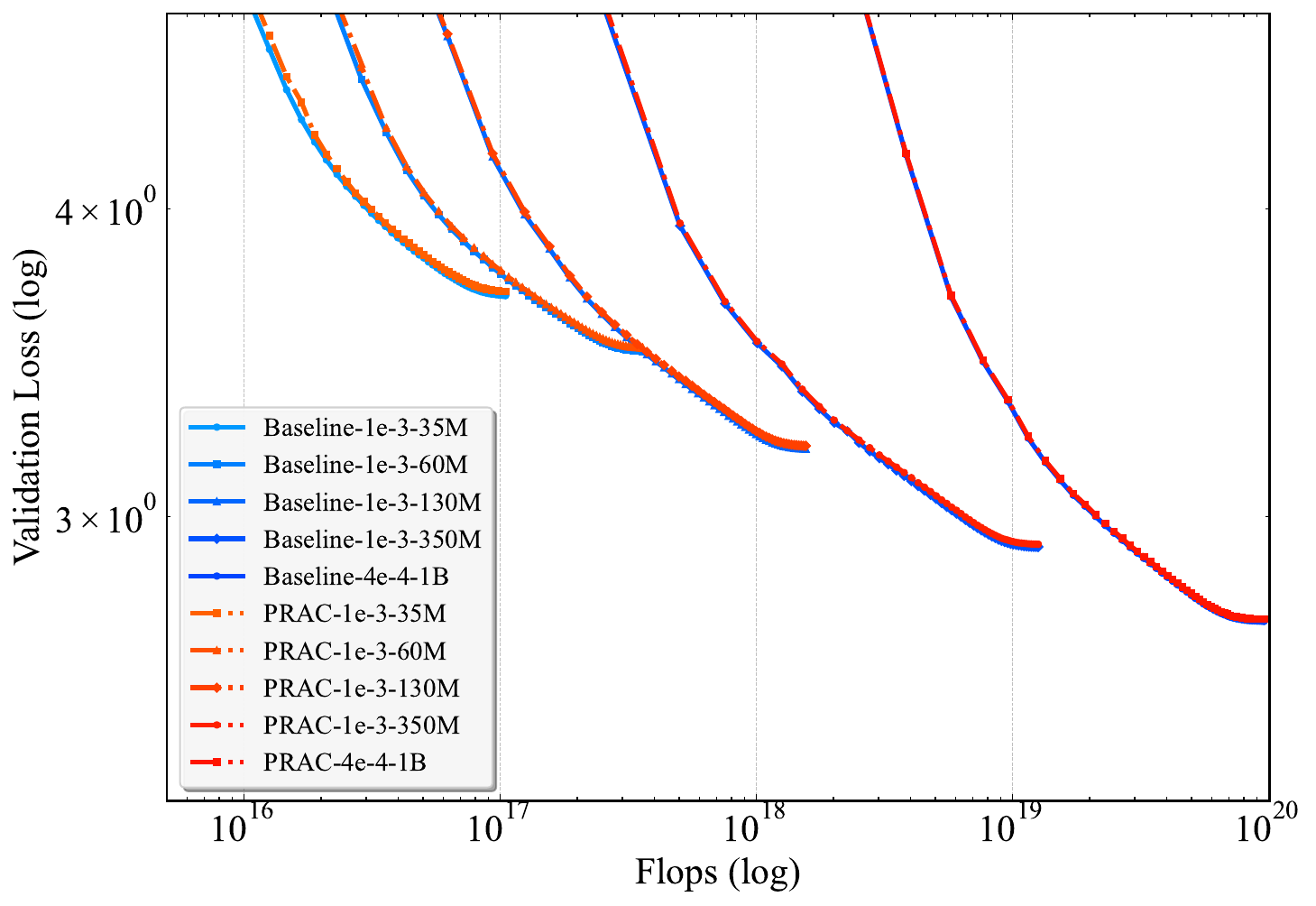}&
        \includegraphics[height=0.23\textwidth]{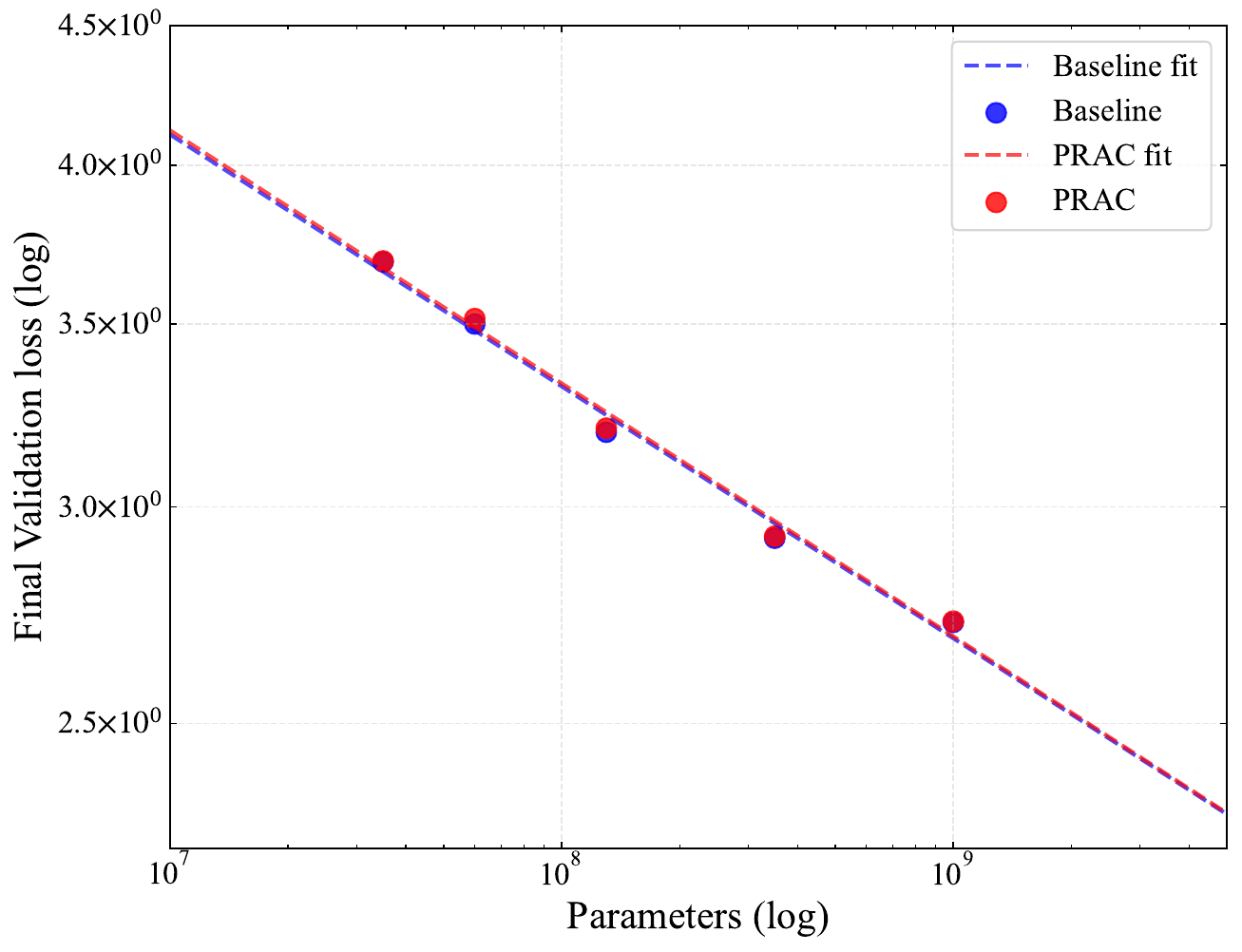}&
        \\
        (a) Scaling laws in terms of compute & (b) Scaling laws in terms of parameters\\
    \end{tabular}
    \end{center}
    \vspace{-0.4cm}
    \caption{Scaling behavior of PRAC across model sizes (35M to 1B parameters)}
    \vspace{-0.5cm}
    \label{fig:scaling_law}
\end{figure}

\begin{table*}[!t]
    \caption{Comparison with memory-efficient algorithms on fine-tuning RoBERTa models. Average scores across the GLUE benchmark and the peak memory usage of activations in MRPC task are provided.}
    \vspace{-0.8cm}
    \label{finetune_roberta}
    \setlength{\tabcolsep}{3pt}
    \begin{center}
    \begin{small}
    \begin{tabular}{l|c|cccccccc|c}
    \toprule
      Method & Memory (GB) & COLA & STSB  & MRPC & RTE & SST2 & MNLI & QNLI & QQP & AVG\\
    \cmidrule{1-11}
    Full Fine-Tuning    & 0.42 & 62.24 & 90.92 & 91.30 & 79.42 & 94.57 & 87.18 & 92.33 & 92.28 & 86.28\\
    \cmidrule{1-11}
    RSO (rank=4)     &0.29& 62.47 & 90.62 & 92.25 & 78.70 & \textbf{94.84} & \textbf{86.67} & 92.29 & 90.94 & 86.10 \\
    PRAC (rank=4)    &0.26 ($\bf{\downarrow 38\%}$)& \textbf{63.56} & \textbf{90.83} & \textbf{93.28} & \textbf{78.71} & 94.72 & 86.60 & \textbf{92.35} & \textbf{90.95} &  \textbf{86.38}\\
    \bottomrule
    \end{tabular}
    \end{small}
    \end{center}
    \vspace{-0.6cm}
\end{table*}

\subsection{Memory-Efficient Fine-tuning} We evaluate PRAC on the GLUE benchmark \citep{GLUE} by fine-tuning RoBERTa \citep{RoBERTa}. For a fair comparison with RSO (rank $r=4$), we set $r_1=r_2=2$ for PRAC. As shown in Table \ref{finetune_roberta}, PRAC not only outperforms the RSO baseline but, in certain tasks, surpasses full fine-tuning. This suggests that the random subspace component may act as a beneficial regularizer during fine-tuning by introducing stochasticity.

\subsection{Ablation Study} \label{sec:ablation}

\begin{wrapfigure}{r}{0.5\linewidth}
\centering
\scriptsize
\vspace{-0.2cm}
\begin{tabular}{@{}c@{}c@{}}
\includegraphics[width=0.5\linewidth]{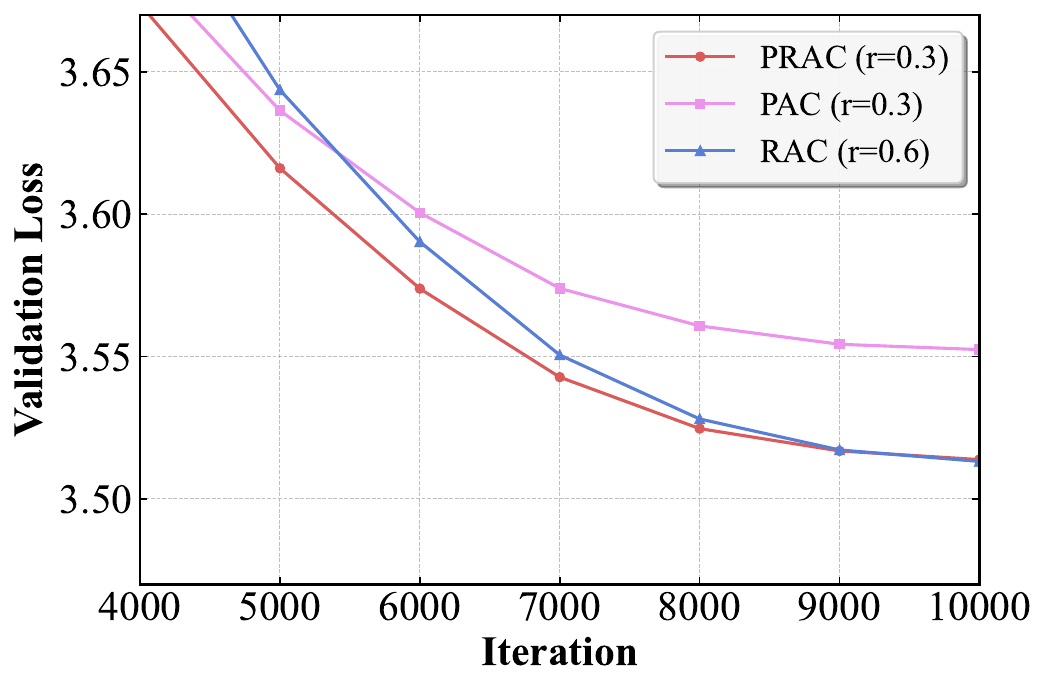}&
\includegraphics[width=0.5\linewidth]{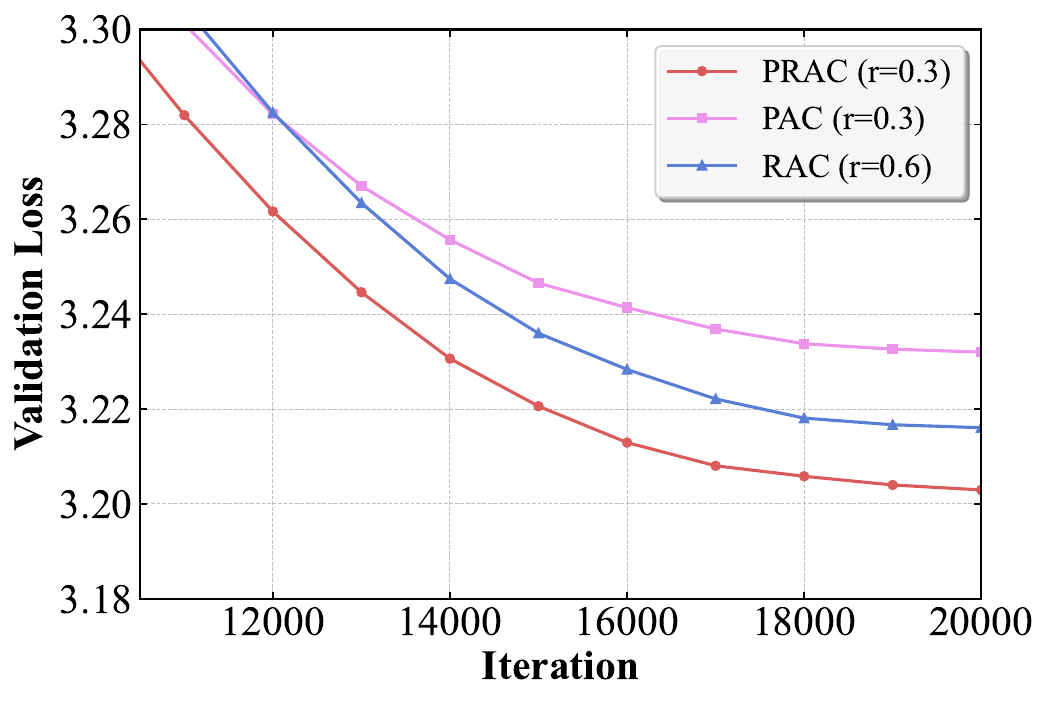}\\
(a) LLaMA 60M & (b) LLaMA 130M
\end{tabular}
\caption{Loss curve of using PRAC, PAC, RAC in LLaMA series pre-training. RAC reports the result of $r=0.6$ due to the divergence of $r=0.3$.}
\label{fig:ablation_dual}
\vspace{-0.6cm}
\end{wrapfigure}
We assess the efficacy of the hybrid projection by comparing PAC, RAC, and PRAC with equivalent total rank (using $r_1+r_2$ for PRAC). As shown in Figure \ref{fig:ablation_dual}, PAC converges slowly in the later stages, where gradient noise dominates the true gradient (Proposition \ref{proposition_1}). In contrast, RAC struggles in the early stages due to its inability to capture principal subspace properties. PRAC overcomes these limitations, achieving faster convergence throughout the entire training process and validating the advantage of combining dual subspaces.

Appendix \ref{additional result} details the ablation analysis for the scaling factor. Results indicate that the theoretical setting $k=\frac{n-r_1}{r_2}$ is effective across both linear and non-linear layers.

\presec
\section{Conclusion}
\postsec
In this work, we propose PRAC, a theoretically optimal activation compression strategy that provides unbiased estimates and leverages the low-rank structure of LLM activations to minimize estimation variance. Unlike prior subspace methods, PRAC ensures stable convergence while reducing total memory by up to 36\% across both pre-training and fine-tuning tasks. With negligible computational overhead, PRAC offers a robust and scalable solution for training large language models on memory-constrained hardware.

\presec
\section*{Impact Statement}
\postsec
This paper presents a method to substantially improve the memory and computational efficiency of LLM training. By compressing activations via principal and random subspaces, our approach reduces the hardware requirements for pre-training and fine-tuning, thereby minimizing the energy consumption and carbon emissions associated with large-scale deep learning workloads.

\vskip 0.2in
\bibliography{sample}
\newpage

\appendix

\presec
\section{Implementation of PRAC}\label{sec:implementation}
\postsec
We present the PRAC method for compressing activations in Algorithm \ref{alg:pq}, and introduce its application during forward and backward propagation in Algorithm \ref{alg:activation_recovery}.
\begin{algorithm*}[!ht]
    \caption{Activation Storage and Recovery for Backpropagation}
    \label{alg:activation_recovery}
    \begin{small}
    \begin{algorithmic}[1]
        \REQUIRE Input $X$, parameters $W$, forward function $f$, ranks ($r_1$, $r_2$), step intervals ($T_1$, $T_2$), current training step $t$
        
        \STATE \textbf{Forward Propagation:}
        \STATE Compute output: $Y \gets f(X; W)$ 
        \STATE Store compression components: 
        $\{Q_1^{(t)}, Q_2^{(t)}, \mathcal{X}_1^{(t)}, \mathcal{X}_2^{(t)}\} \gets \mathcal{C}(X, r_1, r_2, T_1, T_2, t)$ (Algorithm \ref{alg:pq})
        \STATE \textbf{Return:} $Y$
        
        \STATE
        \STATE \textbf{Backward Propagation:}
        \STATE Receive gradient w.r.t output: $\nabla_Y \mathcal{L} \gets \frac{\partial \mathcal{L}}{\partial Y}$
        \STATE Reconstruct input approximation from cache: 
        $\tilde{X} \gets \mathcal{X}_1^{(t)}Q_1^{(t)}+\mathcal{X}_2^{(t)} Q_2^{(t)}$
        \STATE Compute parameter and input gradients: 
        $\nabla_X \mathcal{L},\nabla_W \mathcal{L} \gets \text{backward}(\nabla_Y \mathcal{L}, \tilde{X}, W)$
     
        \STATE \textbf{Return:} $\nabla_W \mathcal{L}, \nabla_X \mathcal{L}$
    \end{algorithmic}
    \end{small}
\end{algorithm*}

\presec
\section{Memory Complexity Analysis}\label{sec:complexity}
\postsec
We take GPT structure as an example to show theoretical activation memory of PRAC. We assume a uniform rank ratio $r_1,r_2\geq 0$ for the principal and random subspaces, respectively. Usually, $r_1+r_2\leq 0.6$.
\begin{table}[!ht]
\centering
\caption{Comparison of activation stored in the GPT series architecture, with decreased activation values highlighted in red for PRAC.}
\label{tab:tensor_comparison_gpt}
\renewcommand{\arraystretch}{1.5}
\begin{tabular}{|c|c|c|}
\hline
\multirow{2}{*}{Operation} & \multicolumn{2}{c|}{Activation Saved} \\ \cline{2-3}
 & Baseline & PRAC $(0\leq r_1+r_2\leq 0.6)$ \\
\hline
\( X = \mathrm{LayerNorm}(\tilde{X}) \) & 
\makecell{\( \tilde{X}\in \mathbb{R}^{b \times s \times n} \)\\
          \( (\mu(\tilde{X}), \sigma(\tilde{X})^2) \in \mathbb{R}^{b \times s} \)} & 
\makecell{\( \textcolor{red}{\tilde{X} P_x\in \mathbb{R}^{b \times s \times (nr_1)},\tilde{X} Q_x\in \mathbb{R}^{b\times s\times (nr_2)}}\)\\
          \( (\mu(X), \sigma(X)^2) \in \mathbb{R}^{b \times s} \)} \\
\hline
\makecell{\( q = X W_q\ \)\\
\( k = X W_k \)\\
\( v = X W_v \)} & 
\makecell{\( X \in \mathbb{R}^{b \times s \times n} \)} & 
\makecell{\( \textcolor{red}{XP_q \in \mathbb{R}^{b\times s \times (nr_1)}} \)\\
          \( \textcolor{red}{XQ_q \in \mathbb{R}^{b\times s \times (nr_2)}} \)} \\
\hline
\makecell{reshape \( q, k, v \)\\
          to \((b, h, s, n/h)\)} & 
None & 
None \\
\hline
\( A_h = \mathrm{flash\_attn}(q, k, v) \) & 
\makecell{\( q, k, v \in \mathbb{R}^{b \times h \times s \times n/h} \)\\
          two buffers \( \in \mathbb{R}^{b \times s \times h} \)} & 
\makecell{\( q, k, v \in \mathbb{R}^{b \times h \times s \times n/h} \)\\
          two buffers \( \in \mathbb{R}^{b \times s \times h} \)} \\
\hline
reshape \( A_h \) to \((b, s, n)\) & 
None & 
None \\
\hline
\( A_o^{\prime\prime} = A_h W_o^T \) & 
\makecell{\( A_h \in \mathbb{R}^{b \times s \times n} \)} & 
\makecell{\( \textcolor{red}{A_h P_h \in \mathbb{R}^{b \times s \times (nr_1)}} \)\\
          \( \textcolor{red}{A_h Q_h \in \mathbb{R}^{b \times s \times (nr_2)}} \)} \\
\hline
\makecell{residual: \( A_o^\prime = A_o^{\prime\prime} + \tilde{X} \)} & 
None & 
None \\
\hline
\( A_o = \mathrm{LayerNorm}(A_o^\prime) \) & 
\makecell{\( A_o^\prime \in \mathbb{R}^{b \times s \times n} \)\\
          \( (\mu(A_o^\prime), \sigma(A_o^\prime)^2) \in \mathbb{R}^{b \times s} \)} & 
\makecell{\( \textcolor{red}{A_o^\prime P_o^\prime \in \mathbb{R}^{b \times s \times (nr_1)}, A_o^\prime Q_o^\prime \in \mathbb{R}^{b \times s \times (nr_2)}} \)\\
          \( (\mu(A_o^\prime), \sigma(A_o^\prime)^2) \in \mathbb{R}^{b \times s} \)} \\
\hline
\(A_{u}^\prime = A_o W_{\mathrm{up}}^T\)  & 
\makecell{\( A_o \in \mathbb{R}^{b \times s \times n} \)} & 
\makecell{\( \textcolor{red}{A_o P_o\in \mathbb{R}^{b\times s \times (nr_1)}} \)\\
          \( \textcolor{red}{A_o Q_o\in \mathbb{R}^{b\times s \times (nr_2)}} \)} \\
\hline
\( A_{u} = \mathrm{GeLU}(A_u^\prime) \) & 
\makecell{\(  A_{u}^\prime \in \mathbb{R}^{b \times s \times m} \)} & 
\makecell{\(  \textcolor{red}{A_{u}^\prime P_{u}^\prime \in \mathbb{R}^{b \times s \times (mr_1)}} \)\\ \(  \textcolor{red}{A_{u}^\prime Q_u^\prime \in \mathbb{R}^{b \times s \times (mr_2)}} \)} \\
\hline
\( A_d^\prime = A_u W_{\mathrm{down}}^T \) & 
\makecell{\( A_u \in \mathbb{R}^{b \times s \times m} \)} & 
\makecell{\( \textcolor{red}{A_u P_u \in \mathbb{R}^{b \times s \times (mr_1)}}\)\\ \( \textcolor{red}{A_u Q_u \in \mathbb{R}^{b \times s \times (mr_2)}} \)} \\
\hline
\makecell{residual: \( A_d = A_d^\prime + A_o^\prime \)} & 
None & 
None \\
\hline
\end{tabular}
\end{table}

A comparison of activation compression between PRAC and baseline methods is presented in Table \ref{tab:tensor_comparison_gpt}. The proposed PRAC compresses both linear layers (including the projection layers in Attention and the MLP) and non-linear layers (including LayerNorm and the GeLU activation). The mean and variance results in LayerNorm are not compressed, as they contribute negligibly to the total memory footprint (i.e., 
$bs\ll bsn$). Flash Attention operations are not compressed by PRAC due to their intricate coupled structure and engineering optimizations. Nevertheless, the compressions above are sufficient for PRAC to achieve an overall memory reduction of nearly 40\% in large-batch-size training tasks.

\presec
\section{Theoretical Results for PRAC}
\postsec

\subsection{Notations and Useful Lemma}
\begin{definition}[Orthogonal Group]
    For a positive integer $m$, let $M_m(\mathbb{R})$ denote the set of all $m\times m$ real matrices. The orthogonal group, denoted 
    $\mathcal{O}(m)$, is defined as:
    \begin{equation}
    \mathcal{O}(m)\coloneqq \left\{A\in M_m(\mathbb{R}) \mid A^\top A=I_m\right\}.
    \end{equation}
\end{definition}
\begin{lemma}\label{lem_1}
Let $\tilde{S}\in \mathbb{R}^{m\times k}$ be a random matrix whose entries are independent and identically distributed (i.i.d.) standard normal variables, i.e., $\tilde{S}_{ij} \stackrel{i.i.d.}{\sim} \mathcal{N}(0,1)$. Consider the (thin) QR decomposition of $\tilde{S}=VR$,
where $V\in \mathbb{R}^{m\times k}$ satisfies $V^\top V=I_k$ and $R\in \mathbb{R}^{k\times k}$ is an upper-triangular matrix with positive diagonal entries. Then for any orthogonal matrix $O\in \mathcal{O}(m)$, $OV$ and $V$ are identically distributed, i.e.
\begin{equation}
OV \stackrel{d}{=} V \;\; \forall O\in \mathcal{O}(m).
\end{equation}
\end{lemma}
\begin{proof}
Let $O\in\mathcal{O}(m)$ be an arbitrary fixed orthogonal matrix and set $\tilde{S}^\prime=O\tilde{S}$. Consider the QR decomposition of $\tilde{S}$ and $\tilde{S}^\prime$, we have
\begin{equation}\label{s'=v'r'}
\tilde{S}=VR,\tilde{S}^\prime=V^\prime R^\prime,
\end{equation}
where $V,V^\prime \in \mathbb{R}^{m\times k}$ satisfie $V^\top V=I_k, (V^\prime)^\top V^\prime=I_k$ and $R,R^\prime$ are upper-triangular matrices with positive diagonal entries. Since the standard normal distribution is rotationally invariant, $\tilde{S}^\prime$ has the same distribution as $\tilde{S}$, i.e., $\tilde{S}^\prime\stackrel{d}{=} \tilde{S}$ , therefore we have $V\stackrel{d}{=}V^\prime$. 
What's more,
\begin{equation}\label{s'=ovr}
\tilde{S}^\prime = O\tilde{S}=OVR.
\end{equation}
Comparing (\ref{s'=v'r'}) and (\ref{s'=ovr}), from the uniqueness of the thin QR decomposition we have $V^\prime = OV,R^\prime =  R$, then
\begin{equation}
V^\prime = OV \stackrel{d}{=}V \;\; \forall O\in \mathcal{O}(m).
\end{equation}
\end{proof}
\begin{lemma}\label{lem_2}
Suppose the matrix $M\in \mathbb{R}^{m\times m}$ satisfies $MO=OM \;\; \forall O\in\mathcal{O}(m)$, then $M$ must be a scalar matrix (i.e., $M=cI_m$).
\end{lemma}
\begin{proof}
    We construct the reflection matrix as follows:
\begin{equation}
J_i=\mathrm{diag}(1,\cdots,1,\underbrace{-1}_i,1,\cdots,1)\in\mathbb{R}^{m\times m},
\end{equation}
where only the $i$-th diagonal entry is -1 and the others are 1. Clearly, $J_i\in \mathcal{O}(m)$. From the commutativity condition 
$J_k$, compare the entries in the $i$-th row and $j$-th column (with 
$i\neq j$), we have
\begin{equation}
    (J_i M)_{ij}=-M_{ij}, (MJ_k)_{ij}=M_{ij}.
\end{equation}
Equating the two expressions gives:
\[
    -M_{ij}=M_{ij} \Rightarrow M_{ij}=0 \;(i\neq j).
\]
Therefore, $M$ is a diagonal matrix. Denote it as:
\begin{equation}
    M=\mathrm{diag}(\lambda_1,\lambda_2,\cdots,\lambda_m).
\end{equation}
Then consider a permutation matrix $P_{ij}$ that swaps the $i$-th and $j$-th rows (and columns). It is also orthogonal, with $P_{ij}^{-1}=P_{ij}^\top=P_{ij}$. From the commutativity condition $P_{ij}M=MP_{ij}$, we have:
\begin{equation}
    M = P_{ij}MP_{ij}.
\end{equation}
Conjugating $M$ by $P_{ij}$ swaps the $i$-th and $j$-th diagonal entries of $M$. Since the matrices $M$ and $P_{ij}MP_{ij}$ are equal, their diagonal entries must coincide, giving $\lambda_i=\lambda_j$ for any pair $i,j$. Hence, $M$ must be a scalar matrix (i.e., $M=cI_m$).
\end{proof}
\begin{lemma}\label{lem_3}
    Suppose the random matrix $V\in \mathbb{R}^{m\times k}$ satisfies $V^\top V=I_k$ and $\forall O\in \mathcal{O}(m),OV\stackrel{d}{=}V$. Then we have 
    \begin{equation}
    \mathbb{E}[VV^\top]=\frac{k}{m}I_m.
    \end{equation}
\end{lemma}
\begin{proof}
For any orthogonal matrix $O\in \mathcal{O}(m)$, $OV$ has the same distribution as $V$. Then, 
\[
\mathbb{E}[VV^\top]=\mathbb{E}[(OV)(OV)^\top]=O\mathbb{E}[VV^\top]O^\top,
\]
due to $O^\top =O^{-1}$,
\[
\mathbb{E}[VV^\top]O=O\mathbb{E}[VV^\top] \;\; \forall O\in\mathcal{O}(m).
\]

From Lemma \ref{lem_2}, the expectation $\mathbb{E}[VV^\top]$ of the random matrix must be a scaler matrix, i.e., 
\begin{equation}
\mathbb{E}[VV^\top]=cI_{m}.
\end{equation}
The constant $c$ is determined by computing the trace:
\begin{equation}
    \mathrm{tr}(\mathbb{E}[VV^\top])=\mathbb{E}[\mathrm{tr}(VV^\top)]=\mathbb{E}[\mathrm{tr}(V^\top V)]=k.
\end{equation}
Meanwhile, $\mathrm{tr}(cI_m)=cm$, thus $cm=k\Rightarrow c=\frac{k}{m}$.
\end{proof}
\begin{lemma}\label{lem_4}
    Suppose $Q_1\in \mathbb{R}^{n\times r_1}$ be a fixed orthogonal matrix, random matrix $S_o$ be generated by (\ref{eqn:so_prac}), and let the random matrix $Q_2\in \mathbb{R}^{n\times r_2}$ be the first component of $\text{QR}(S_o)$, then we have:
    \begin{equation}
    \mathbb{E}[Q_2Q_2^\top]=\frac{r_2}{n-r_1}(I_n-Q_1Q_1^\top).
    \end{equation}
\end{lemma}
\begin{proof}
    Let $U\in \mathbb{R}^{n\times (n-r_1)}$ be an orthonormal basis for the orthogonal complement of $P$:
    \begin{equation}
    U^\top U= I_{n-r_1}, Q_1^\top U =O_{r_1\times(n-r_1)}, Q_1Q_1^\top+UU^\top = I_n,
    \end{equation}
    then
    \begin{equation}
    S_o = S-Q_1Q_1^\top S = UU^\top S.
    \end{equation}
    We define $\tilde{S}=U^\top S\in \mathbb{R}^{(n-r_1)\times r_2}$, then $S_o=U\tilde{S}$. Since $U$ has orthonormal columns and the entries of $S$ are i.i.d Gaussian, the entries of $\tilde{S}$ are also i.i.d Gaussian. Perform QR decomposition on $\tilde S$: $\tilde{S}=VR$, then
    \begin{equation}
    S_o=U\tilde{S}=(UV)R,
    \end{equation}
    according to the uniqueness of the thin QR decomposition, the projection matrix in the algorithm satisfies $Q_2=UV\in\mathbb{R}^{n\times r_2}$.
    Then, 
    \begin{equation}
        \begin{aligned}
            \mathbb{E}[Q_2Q_2^\top]=\mathbb{E}[UVV^\top U^\top]&=U\mathbb{E}[VV^\top ]U^\top \\
            &\stackrel{(a)}{=}U(\frac{r_2}{n-r_1}I_{n-r_1})U^\top\\
            &=\frac{r_2}{n-r_1}(I_n-Q_1Q_1^\top)
        \end{aligned}
    \end{equation}
    where the equality (a) holds according to Lemma \ref{lem_3} and $V\in \mathbb{R}^{(n-r_1)\times r_2}$.
\end{proof}

\subsection{Omitted Proofs in Section \ref{sec:proposed_prac}}\label{omit_2}

\begin{theorem}[Unbiased Reconstruction of PRAC]\label{thm:unbiased_reconstruction_} Let $P\in \mathbb{R}^{n\times r_1}$ be the projection matrix defined in \eqref{eqn:R}. If the scaling factor is set to $k=\frac{n-r_1}{r_2}$, the reconstruction $\tilde{X}=XPP^\top$ satisfies $\mathbb{E}[\tilde{X}]=X$. \end{theorem}
\begin{proof}
    The expectations for the reconstruction $\tilde{X}$ are as follows:
    \begin{equation}
        \begin{aligned}
            \mathbb{E}[\tilde{X}]=\mathbb{E}[XQ_1Q_1^T+kXQ_2Q_2^\top]&=XQ_1Q_1^\top+kX\mathbb{E}[Q_2Q_2^\top]\\
            &\stackrel{(a)}{=}XQ_1Q_1^\top+k\frac{r_2}{n-r_1}X(I_n-Q_2Q_2^\top)\\
            &=XQ_1Q_1^\top+\frac{n-r_1}{r_2}\frac{r_2}{n-r_1}X(I_n-Q_2Q_2^\top)\\
            &=X,
        \end{aligned}
    \end{equation}
    where in (a) we use Lemma \ref{lem_4} to replace the $\mathbb{E}[Q_2Q_2^\top]$ term.
\end{proof}

\begin{theorem}[Unbiased Reconstruction of RAC]\label{thm:unbiased_reconstruction_of_rac_}
Let $Q_2\in \mathbb{R}^{n\times r_2}$ sample uniformly from the Stiefel manifold. Set the scaling factor $k=\frac{n}{r_2}$. Then the reconstruction $\tilde{X}_{\text{rac}}=(k XQ_2)Q_2^\top$ satisfies $\mathbb{E}[\tilde{X}_{\text{rac}}]=X$.
\end{theorem}
\begin{proof}
This result follows as a special case of Theorem \ref{thm:unbiased_reconstruction_} by setting $r_1=0$ and then $k = n/r_2$.
\end{proof}

\begin{theorem}[Variance Bound of PRAC]\label{thm:gradient_variance_} Under the conditions of Theorem \ref{thm:unbiased_reconstruction_}, the reconstruction error variance is bounded by: \[
\mathbb{E}\bigl[\|\tilde{X} - X\|_F^2\bigr] \leq (\frac{n-r_1}{r_2} - 1)\,\bigl\|X - XQ_1Q_1^{\top}\bigr\|_F^2.
\]\end{theorem}

\begin{proof}
    We evaluate the expected squared Frobenius norm of the reconstruction error,
    \begin{equation}
        \begin{aligned}
            \mathbb{E}[\|X - \tilde{X}\|_F^2] &= \mathbb{E}[\|X(I_n - Q_1Q_1^\top - k Q_2Q_2^\top)\|_F^2].
        \end{aligned}
    \end{equation}
    Recall from the properties of $Q_1$ that $I_n - Q_1Q_1^\top = UU^\top$. Substituting this relation, we obtain
    \begin{equation}\label{ex-x'2}
        \begin{aligned}
        \mathbb{E}[\|X - \tilde{X}\|_F^2] = \mathbb{E}[\|X(UU^\top - k Q_2Q_2^\top)\|_F^2]= \mathbb{E}[\operatorname{tr}(X(UU^\top - k Q_2Q_2^\top)^2 X^\top)],
        \end{aligned}
    \end{equation}
    where the last equality follows from the identity $\|A\|_F^2 = \operatorname{tr}(AA^\top)$. Then, we compute the expectation of the squared matrix term:
    \begin{equation}
        \mathbb{E}\left[(UU^\top-kQ_2Q_2^\top)^2\right]=UU^\top-2k\mathbb{E}[Q_2Q_2^\top]+k^2\mathbb{E}[Q_2Q_2^\top]
    \end{equation}
    From Lemma \ref{lem_3}, we have $\mathbb{E}[Q_2Q_2^\top] = \frac{1}{k} UU^\top$. Substituting this yields
    \begin{equation}\label{euu-kqqt2}
        \mathbb{E}\left[(UU^\top-kQ_2Q_2^\top)^2\right]=UU^\top(1-2k\frac{1}{k}+k^2\frac{1}{k})=UU^\top(k-1),
    \end{equation}
    Substituting (\ref{euu-kqqt2}) into (\ref{ex-x'2}), we have:
    \begin{equation}\label{ex-x'2-final}
        \mathbb{E}[\|X-\tilde{X}\|_F^2]=\mathrm{tr}(XUU^\top X^\top)(k-1)=(\frac{n-r_1}{r_2}-1)\|X-XQ_1Q_1^\top\|_F^2
    \end{equation}
\end{proof}

\begin{theorem}[Variance Bound of RAC]\label{thm:gradient_variance_of_rac_} Under the conditions of Theorem \ref{thm:unbiased_reconstruction_of_rac_}, then the reconstruction error variance is bounded by: \[
\mathbb{E}\bigl[\|\tilde{X}_{\text{rac}} - X\|_F^2\bigr] \leq (\frac{n}{r_2} - 1)\,\|X\|_F^2.
\]\end{theorem}

\begin{proof}
This result follows as a special case of Theorem \ref{thm:gradient_variance_} by setting $r_1=0$ and then $k = n/r_2$.
\end{proof}

\begin{theorem}[Lower Bound]\label{thm:lower bound}
    Under Assumption \ref{ass:spectral_decay}, if $s < r < n$, then
    \[
    \eqref{eq:lower bound} \ge \left(\frac{n-s}{r-s}-1\right)q.
    \]
\end{theorem}

\begin{proof}
    Without loss of generality, assume the activation matrix is $X = \operatorname{diag}(\sigma_1,\dots,\sigma_n)\in\mathbb{R}^{n\times n}$ with $\sigma_1\ge\sigma_2\ge\cdots\ge\sigma_n>0$ and, by Assumption \ref{ass:spectral_decay}, $\sum_{i=s+1}^{n}\sigma_i^2\le q$. The unbiasedness condition $\mathbb{E}[XPP^\top]=X$ implies $\mathbb{E}[PP^\top]=I$. Set $A=PP^\top\in\mathbb{R}^{n\times n}$; then the expected error can be written as
    \begin{equation}
    \epsilon = \mathbb{E}\|X(A-I)\|_F^2 = \sum_{i=1}^{n}\sigma_i^2\Bigl[\mathbb{E}(A_{ii}-1)^2 + \sum_{j\neq i}\mathbb{E}A_{ij}^2\Bigr] \equiv \sum_{i=1}^{n}\sigma_i^2\alpha_i,
    \end{equation}
    where $\alpha_i\ge0$ is defined by the bracket.

    \begin{enumerate}
        \item \textbf{Vanishing of the first $s$ coefficients.} Consider a family of $X$ where $\sigma_i=t$ for $i\le s$ while $\sigma_{s+1},\dots,\sigma_n$ are fixed and satisfy the sum‑of‑squares constraint. If for some distribution of $P$ we have $\sum_{i=1}^{s}\alpha_i>0$, then sending $t\to\infty$ would make $\epsilon$ arbitrarily large, so $\sup_X\inf_P\epsilon = \infty$ and the bound holds trivially. Hence we need only consider strategies with $\sum_{i=1}^{s}\alpha_i=0$, i.e. $\alpha_i=0$ for all $i\le s$. This forces $P$ to be block‑diagonal almost surely:
        \begin{equation}
        P = \begin{bmatrix} I_s & 0 \\ 0 & \tilde{P} \end{bmatrix},\qquad \tilde{P}\in\mathbb{R}^{(n-s)\times(r-s)}.
        \end{equation}
        Consequently the error reduces to
        \begin{equation}\label{eq:error_reduced}
        \epsilon = \sum_{i=s+1}^{n}\alpha_i\sigma_i^2.
        \end{equation}
        \item \textbf{Lower bound on the sum of the remaining $\alpha_i$.} Let $\tilde{A}=\tilde{P}\tilde{P}^\top\in\mathbb{R}^{(n-s)\times(n-s)}$. From $\mathbb{E}[A]=I$ we obtain $\mathbb{E}[\operatorname{tr}(\tilde{A})] = \mathbb{E}[\operatorname{tr}(A)]-s = n-s$. Because $\operatorname{rank}(\tilde{A})\le r-s$, Cauchy–Schwarz gives
        \begin{equation}
            \operatorname{tr}(\tilde{A}^2) \ge \frac{[\operatorname{tr}(\tilde{A})]^2}{r-s}.
            \end{equation}
            Taking expectations and using Jensen's inequality,
            \begin{equation}
            \mathbb{E}[\operatorname{tr}(\tilde{A}^2)] \ge \frac{(\mathbb{E}[\operatorname{tr}(\tilde{A})])^2}{r-s} = \frac{(n-s)^2}{r-s}.
            \end{equation}
            Now compute the sum of $\alpha_i$ for $i>s$:
            \begin{equation}\label{eq:sum_alpha_bound}
            \begin{aligned}
            \sum_{i=s+1}^{n}\alpha_i
            &= \sum_{i=s+1}^{n}\Bigl[\mathbb{E}(A_{ii}-1)^2 + \sum_{j\neq i}\mathbb{E}A_{ij}^2\Bigr] \\
            &= \sum_{i=s+1}^{n}\Bigl[\sum_{j=1}^{n}\mathbb{E}A_{ij}^2 - 1\Bigr] \qquad (\text{since }\mathbb{E}A_{ii}=1)\\
            &= \sum_{i=s+1}^{n}\sum_{j=s+1}^{n}\mathbb{E}A_{ij}^2 - (n-s) \qquad (\text{by the block‑diagonal form})\\
            &= \mathbb{E}[\operatorname{tr}(\tilde{A}^2)] - (n-s)\\
            &\ge \frac{(n-s)^2}{r-s} - (n-s) \equiv C.
            \end{aligned}
        \end{equation}
        \item \textbf{Completion of the lower bound.} For any fixed $\sigma_{s+1},\dots,\sigma_n$, using $\alpha_i\ge0$ we have
        \begin{equation}
        \sum_{i=s+1}^{n}\alpha_i\sigma_i^2 \ge \Bigl(\sum_{i=s+1}^{n}\alpha_i\Bigr)\min_{s+1\le i\le n}\sigma_i^2 \ge C\cdot\min_i\sigma_i^2.
        \end{equation}
        Taking the infimum over admissible $\alpha_i$ (i.e. over distributions of $P$) and then the supremum over $\sigma_i$ satisfying $\sum_{i=s+1}^{n}\sigma_i^2\le q$ yields
        \begin{equation}
        \sup_{\sigma_i}\inf_{\alpha_i}\sum_{i=s+1}^{n}\alpha_i\sigma_i^2 \ge C\cdot\sup_{\sigma_i}\min_i\sigma_i^2.
        \end{equation}
        Under the sum‑of‑squares constraint, $\sup\min_i\sigma_i^2$ is attained when all $\sigma_i^2$ are equal, i.e. $\sigma_i^2 = q/(n-s)$; hence $\sup\min_i\sigma_i^2 = q/(n-s)$. Substituting this together with the value of $C$ from \eqref{eq:sum_alpha_bound} gives
        \begin{equation}
        \sup_{\sigma_i}\inf_{\alpha_i}\sum_{i=s+1}^{n}\alpha_i\sigma_i^2 \ge \Bigl(\frac{(n-s)^2}{r-s}-(n-s)\Bigr)\frac{q}{n-s}
        = \Bigl(\frac{n-s}{r-s}-1\Bigr)q.
        \end{equation}
        By \eqref{eq:error_reduced} this is exactly a lower bound for the quantity $\eqref{eq:lower bound}$, completing the proof.
    \end{enumerate}
    
\end{proof}

\presec
\section{Experiment Details}
\postsec
\subsection{Pre-training Setting}\label{sec:Pre-training Setting} We detail the architectural configurations and pre-training hyperparameters for both the LLaMA and GPT models. To ensure numerical stability and computational efficiency, all LLaMA models are trained using bfloat16 precision and GPT models are trained using Automatic Mixed Precision. The key hyperparameters for LLaMA and GPT models across different scales are summarized in Table \ref{tab:pretrain_setting}.

\noindent\textbf{LLaMA Settings.} The LLaMA models are trained with a maximum sequence length of 256 and a global batch size of 512. We employ a learning rate schedule comprising a linear warmup over the first 10\% of training steps, followed by a cosine annealing phase that decays the learning rate to 10\% of its peak value. Across all LLaMA model scales, we adopt the optimal learning rates reported in the original papers for comparison methods. For PRAC, the learning rate is selected from the grid $\left\{0.0006,0.0008,0.001,0.002,0.004\right\}$, which is closely aligned with the settings used by the baselines.

\noindent\textbf{GPT Settings.} For the GPT models, we adopt a maximum sequence length of 1024 while maintaining a global batch size of 512. We set the warmup steps to 2K, which corresponds to 4\% of the total steps\footnote{\url{https://github.com/zyushun/Adam-mini/tree/main}}. Notably, given that the official code and hyperparameters for GaLore \citep{galore2024} and RSO \citep{rso2025} on GPT models are not publicly available, we train these models using learning rates aligned with the baseline. For these methods, we configure the ranks as 256 and 512 for the 124M and 355M models, respectively.

\begin{table}[!ht]
    \vspace{-0.4cm}
    \caption{Hyperparameter settings for pre-traing LLaMA and GPT-2 model.}
    \vspace{0.3cm}
    \label{tab:pretrain_setting}
    \setlength{\tabcolsep}{5pt}
    \begin{center}
    \begin{small}
    \begin{tabular}{lcccccccc}
    \toprule
    & Params & Hidden & Intermediate & Heads & Layers & Steps\\
    \cmidrule{1-7}
    \multirow{5}*{\tabincell{c}{LLaMA}}
    &35M & 384 & 1024 & 8 & 6 & 5K\\
    &60M & 512 & 1376 & 8 & 8 & 10K\\
    &130M & 768 & 2048 & 12 & 12 & 20K\\
    &350M & 1024 & 2736 & 16 & 24 & 60K\\
    &1B & 2048 & 5461 & 24 & 32 & 150K \\
    \cmidrule{1-7}
    \multirow{2}*{\tabincell{c}{GPT-2}}
    &124M & 768 & 3072 & 12 & 12 & 50K\\
    &355M & 1024 & 4096 & 16 & 24 & 50K\\
    \bottomrule
    \end{tabular}
    \end{small}
    \end{center}
    \vspace{-0.5cm}
\end{table}
\subsection{Fine-tuning Setting}\label{sec:Fine-tuning Setting}
We fine-tune the pre-trained RoBERTa-Base model on the GLUE benchmark using the Hugging Face library\footnote{\url{https://huggingface.co/transformers/model_doc/roberta.html}}. All tasks are trained for 30 epochs with a batch size of 16. Refer to Table \ref{finetune_roberta_setting} for the specific hyperparameters employed for PRAC.
\begin{table*}[!ht]
    \vspace{-0.2cm}
    \caption{Hyperparameter settings for fine-tuning RoBERTa-Base model on the GLUE benchmark using the PRAC method.}
    \vspace{-0.2cm}
    \label{finetune_roberta_setting}
    \setlength{\tabcolsep}{3pt}
    \begin{center}
    \begin{small}
    \begin{tabular}{lccccccccc}
    \toprule
       & COLA & STSB  & MRPC & RTE & SST2 & MNLI & QNLI & QQP\\
    \cmidrule{1-9}
    Batch Size        & 16 & 16 & 16 & 16 & 16 & 16 & 16 & 16\\
    Epochs            & 30 & 30 & 30 & 30 & 30 & 30 & 30 & 30\\
    Learing Rate      & 3E-05 & 3E-05 & 3E-05 & 3E-05 & 1E-05 & 1E-05 & 1E-05 & 1E-05\\
    Rank ($r_1+r_2$)  &\multicolumn{8}{c}{4}\\
    PRAC Interval     &\multicolumn{8}{c}{200}\\
    Max Seq Length    &\multicolumn{8}{c}{512}\\
    \bottomrule
    \end{tabular}
    \end{small}
    \end{center}
\end{table*}

\subsection{Additional Result}\label{additional result}
\noindent\textbf{Compatibility with Advanced Optimizers.} PRAC is orthogonal to optimizer choice. We integrate PRAC with memory-efficient optimizers Muon \citep{muon2025} and Adam-mini \citep{adam_mini2024}. As shown in Figure \ref{fig:mix} and Table \ref{tab:mix}, PRAC successfully reduces memory by an additional approximately 30\% on top of these optimizers with negligible perplexity degradation, highlighting its versatility.

\begin{figure*}[!ht]\label{mix}
    \scriptsize
    \setlength{\tabcolsep}{2pt}
    \begin{center}
    \begin{tabular}{ccccc}
        \includegraphics[width=0.23\textwidth]{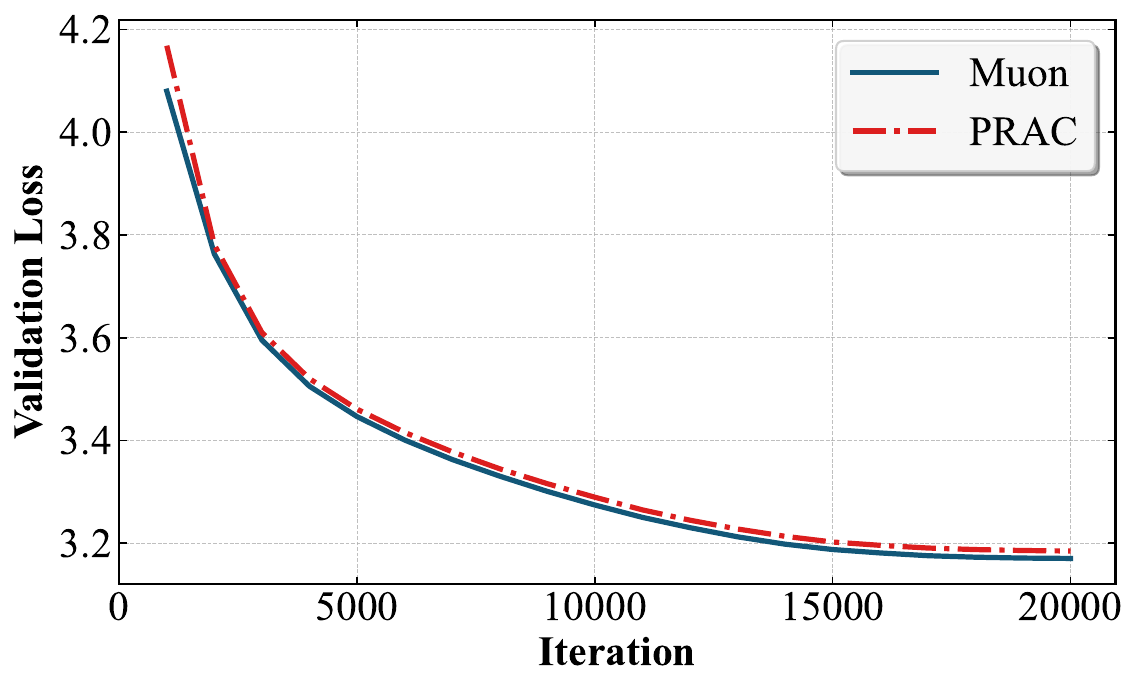}&
        \includegraphics[width=0.23\textwidth]{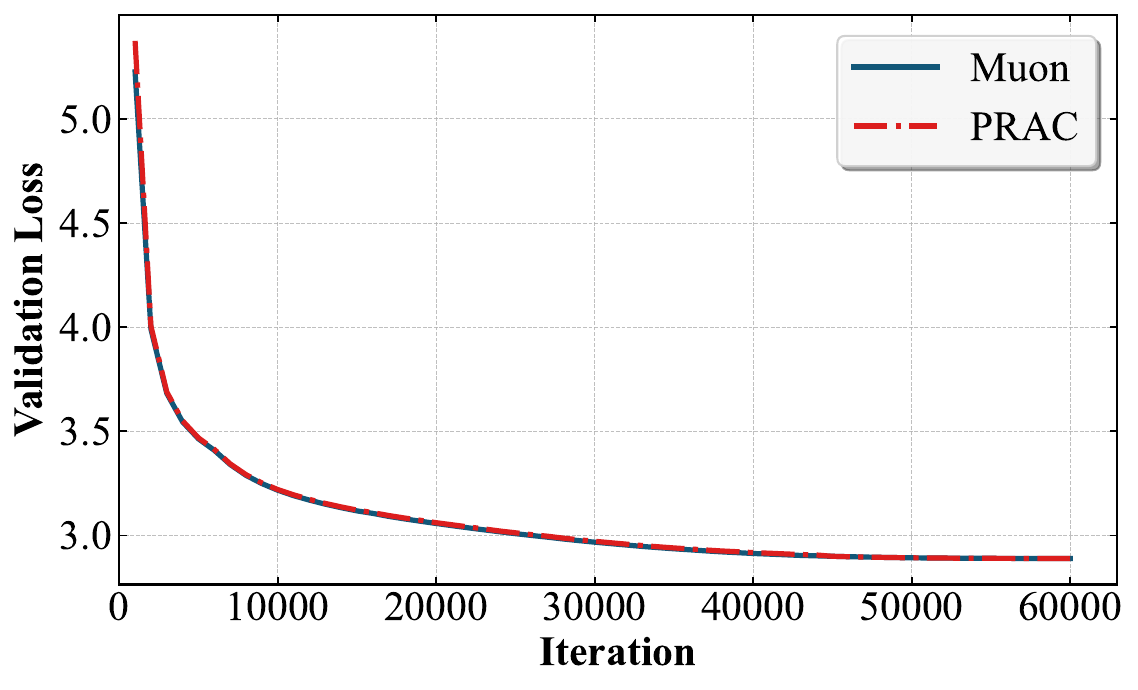}&
        \includegraphics[width=0.23\textwidth]{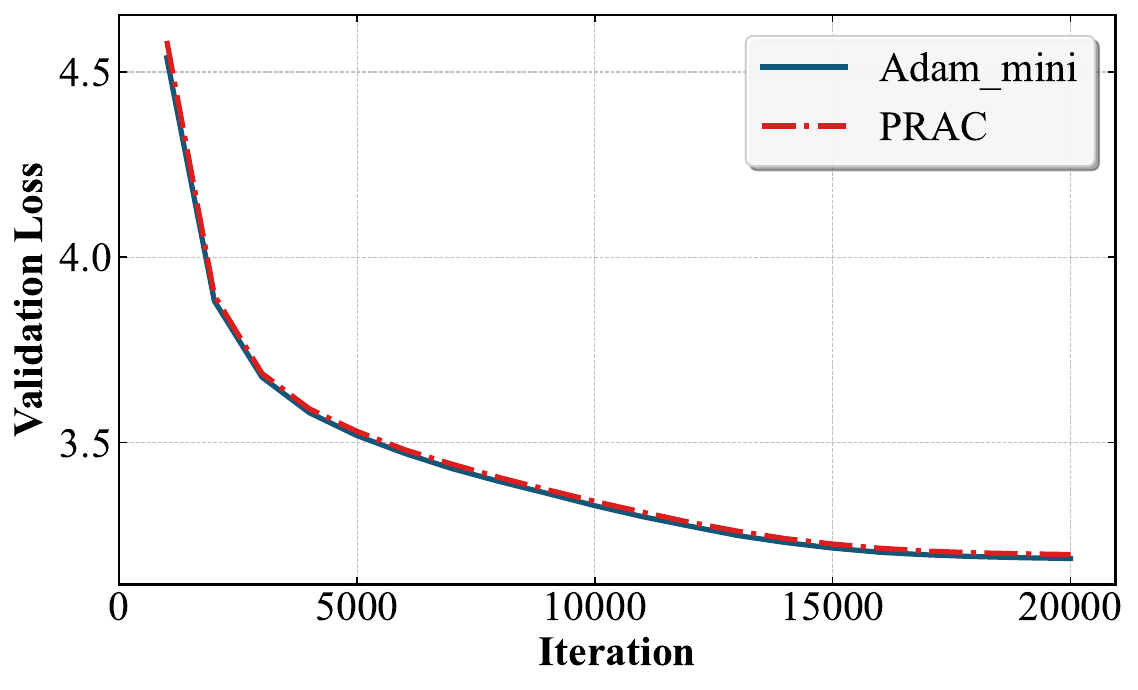}&
        \includegraphics[width=0.23\textwidth]{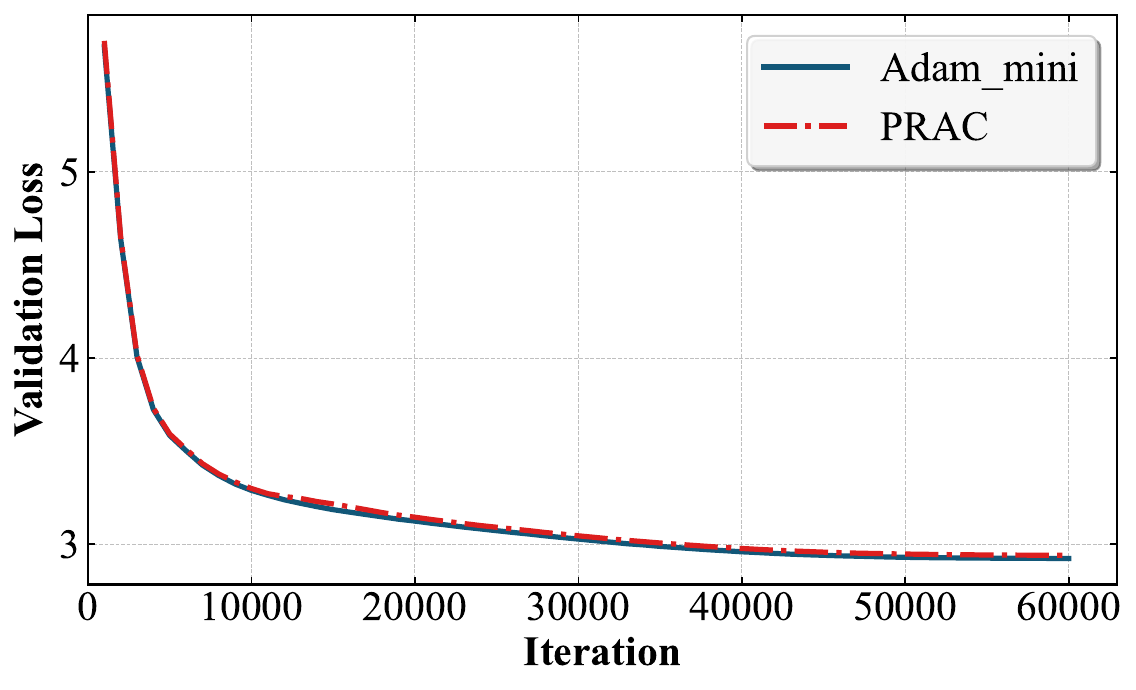}&
        \\
        (a) Muon (LLaMA 130M) & (b) Muon (LLaMA 350M) & (c) Adam-mini (LLaMA 130M) & (d) Adam-mini (LLaMA 350M)\\
    \end{tabular}
    \end{center}
    \vspace{-0.2cm}
    \caption{Loss curves of pre-training LLaMA model based on the Muon and Adam-mini optimizers w./w.o. PRAC.}
    \vspace{-0.5cm}
    \label{fig:mix}
\end{figure*}

\begin{table}[!ht]
    \vspace{-0.3cm}
    \caption{Combining with advanced optimizers on pretraining LLaMA. Report validation perplexity (PPL, lower is better) and the algorithm's peak memory usage (lower is better).}
    \vspace{-0.3cm}
    \label{tab:mix}
    \setlength{\tabcolsep}{4pt}
    \centering
    \begin{small}
    \begin{tabular}{lccccc}
    \toprule
    \multirow{2}{*}{Method} & \multirow{2}{*}{PRAC} & \multicolumn{2}{c}{LLaMA-130M} & \multicolumn{2}{c}{LLaMA-350M} \\
    \cmidrule{3-6}
    & & Ppl & Mem & Ppl & Mem \\
    \midrule
    \multirowcell{2}{Muon} & $\times$ & 23.90 & 21.04 & 17.98 & 45.26 \\
        & $\checkmark$ & 24.16 & 15.37 ($\bf{\downarrow 27\%}$) & 18.04 & 30.82 ($\bf{\downarrow 32\%}$) \\
    \midrule
    \multirowcell{2}{Adam-\\mini} & $\times$ & 24.22 & 21.01 & 18.60 & 45.42 \\
        & $\checkmark$ & 24.48 & 15.36 ($\bf{\downarrow 27\%}$)& 18.92 & 31.09 ($\bf{\downarrow 31\%}$)\\
    \bottomrule
    \end{tabular}
    \end{small}
\end{table}

\noindent\textbf{Selection of the Scaling Factor $k$.} We evaluate different values of $k$ on the linear and normalization layers of LLaMA-130M. Defining the theoretical baseline as $k_0=\frac{n-r_1}{r_2}$, we test $k\in \left\{k_0,0.5k_0,0.2k_0,1.2k_0\right\}$. The results in Figure \ref{fig:ablation_k0} indicate that $k=k_0$  yields the best convergence, whereas other settings result in performance degradation. This empirical evidence aligns with our theoretical findings (See Section \ref{sec:prac}).

\begin{figure}[!ht]
    \scriptsize
    \setlength{\tabcolsep}{2pt}
    \begin{center}
    \begin{tabular}{ccccc}
        \includegraphics[width=0.26\textwidth]{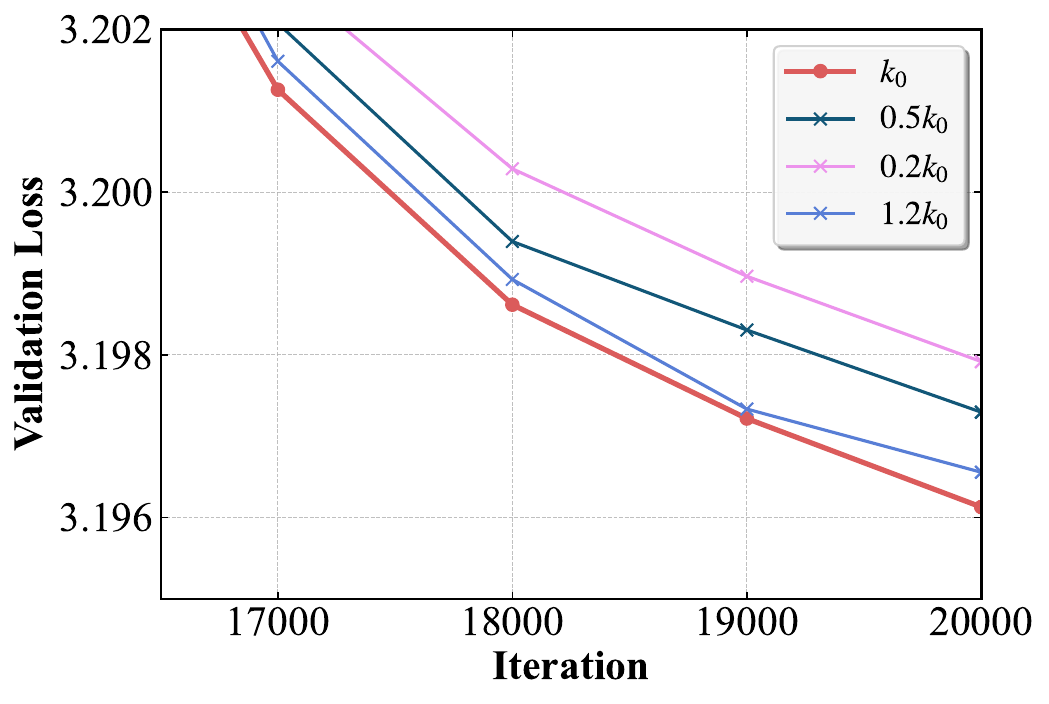}&
        \includegraphics[width=0.26\textwidth]{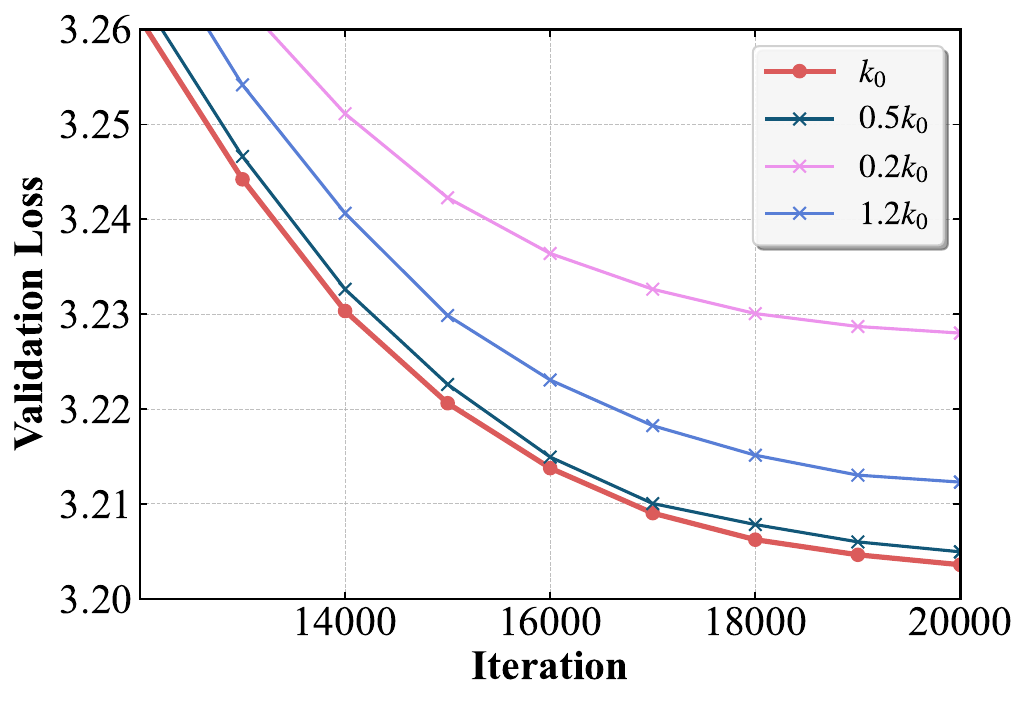}&
        \\
        (a) Linear Layer & (b) Norm Layer\\
    \end{tabular}
    \end{center}
    \vspace{-0.3cm}
    \caption{Loss curve of using different $k$ in LLaMA-130M pre-training, where $k_0=\frac{n-r_1}{r_2}$, Setting 
    $k=k_0$ performs better than larger or smaller settings, whether in linear or nonlinear layers.}
    \vspace{-0.5cm}
    \label{fig:ablation_k0}
\end{figure}

\end{document}